\documentclass{aamas2016}

\usepackage{graphicx}
\usepackage{amsmath, amssymb}
\usepackage{amsfonts}
\usepackage{url}
\usepackage{calc}

\usepackage{graphicx} 
\usepackage{subfigure}
\usepackage{caption}

\usepackage{algorithmic}
\usepackage[vlined,algoruled]{algorithm2e}

\input{Definitions.tex}

\usepackage[colorlinks,linkcolor=red,citecolor=blue,urlcolor=blue]{hyperref}

\newcommand{\xdim}{n}

\newcommand{\G}{G}

\newcommand{\Gunbiased}[1]{G_{#1}}
\newcommand{\Gbiased}[1]{\hat{G}_{#1}}
\newcommand{\Glam}[1]{G^{\lambda}_{#1}}

\newcommand{\e}{\mathbf{e}}
\newcommand{\h}{\mathbf{h}}

\newcommand{\xvec}{\mathbf{x}}

\newcommand{\x}{\mathbf{x}}
\newcommand{\w}{\mathbf{w}}

\newcommand{\herr}{\mathbf{w}^{\text{err}}}
\newcommand{\hvar}{\mathbf{w}^{\text{sq}}}
\newcommand{\hsq}{\hvar}
\newcommand{\auxhsq}{\mathbf{h}^{sq}}

\newcommand{\hsqauxstep}{\bar{\alpha}_{h}}

\newcommand{\Lambdamat}{\boldsymbol{\Lambda}}

\newcommand{\var}{{\mathrm{Var}}}
\newcommand{\defeq}{\overset{\text{def}}{=}}
\newcommand{\assigneq}{\gets}

\newcommand{\oneminuslam}[1]{(1\!-\!\lambda_{#1})}

\newcommand{\As}{\mathcal{A}}
\newcommand{\Ss}{\mathcal{S}}
\newcommand{\E}{\mathbb{E}}
\newcommand{\Rn}{\mathbb{R}}

\newcommand{\glambda}[1]{{\Glam{#1}}}
\newcommand{\glambdasq}[1]{(\Glam{#1})^2}
\newcommand{\gtrace}[1]{{\bar{\e}_{#1}}}
\newcommand{\gtwotrace}[1]{{\bar{\mathbf{z}}_{#1}}}
\newcommand{\gvectracer}[1]{{\bar{\mathbf{a}}_{#1}}}
\newcommand{\gvectracex}[1]{{\bar{\mathbf{B}}_{#1}}}
\newcommand{\glambdasqbar}{\bar{V}^{\lambar}}

\newcommand{\Psq}{\mathbf{\bar{P}}^{\pi,\gamma}}
\newcommand{\Vsq}{\mathbf{\bar{v}}}
\newcommand{\rbarvec}{\mathbf{\rbar}}
\newcommand{\Bellmansq}{\mathbf{\bar{T}}}

\newcommand{\lamname}{$\lambda$-greedy}

\newcommand{\citet}[1]{\cite{#1}}

\newcommand{\varMSPBE}{\text{Var-MSPBE}}

\newcommand{\lambar}{\bar{\lambda}}
\newcommand{\gambar}{\bar{\gamma}}
\newcommand{\deltabar}{\bar{\delta}}
\newcommand{\rbar}{\bar{r}}

\newcommand{\deltalambar}[1]{\bar{\delta}^{\lambar}_{#1}}
\newcommand{\gbar}{\bar{G}}
\newcommand{\ztrace}{\mathbf{\bar{z}}}

\numberofauthors{2}

\author{
\alignauthor
Martha White\\
       \affaddr{Department of Computer Science}\\
       \affaddr{Indiana University}\\
       \affaddr{Bloomington, IN 47405, USA}\\
       \email{martha@indiana.edu}
\alignauthor
Adam White\\
       \affaddr{Department of Computer Science}\\
       \affaddr{Indiana University}\\
       \affaddr{Bloomington, IN 47405, USA}\\
       \email{adamw@indiana.edu}
       }
       
\begin{document}

\title{A Greedy Approach to Adapting the Trace Parameter for Temporal Difference Learning}

\maketitle

\begin{abstract}
One of the main obstacles to broad application of reinforcement learning methods is the parameter sensitivity of our core learning algorithms. In many large-scale applications, online computation and function approximation represent key strategies in scaling up reinforcement learning algorithms. In this setting, we have effective and reasonably well understood algorithms for adapting the learning-rate parameter, online during learning. Such meta-learning approaches can improve robustness of learning and enable specialization to current task, improving learning speed. For temporal-difference learning algorithms which we study here, there is yet another parameter, $\lambda$, that similarly impacts learning speed and stability in practice. Unfortunately, unlike the learning-rate parameter, $\lambda$ parametrizes the objective function that temporal-difference methods optimize. Different choices of $\lambda$ produce different fixed-point solutions, and thus adapting $\lambda$ online and characterizing the optimization is substantially more complex than adapting the learning-rate parameter. There are no meta-learning method for $\lambda$ that can achieve (1) incremental updating, (2) compatibility with function approximation, and (3) maintain stability of learning under both on and off-policy sampling. In this paper we contribute a novel objective function for optimizing $\lambda$ as a function of state rather than time. We derive a new incremental, linear complexity $\lambda$-adaption algorithm that does not require offline batch updating or access to a model of the world, and present a suite of experiments illustrating the practicality of our new algorithm in three different settings. Taken together, our contributions represent a concrete step towards black-box application of temporal-difference learning methods in real world problems.

\end{abstract}

\section{Introduction}

In reinforcement learning, the training data is produced by an adaptive learning agent's interaction with its environment, which makes tuning the parameters of the learning process both challenging and essential for good performance. In the online setting we study here, the agent-environment interaction produces an unending stream of temporally correlated data. In this setting there is no testing-training split, and thus the agent's learning process must be robust and adapt to new situations not considered by the human designer. Robustness is often critically related to the values of a small set parameters that control the learning process (e.g., the step-size parameter). In real-world applications, however, we cannot expect to test a large
range of theses parameter values, in all the situations the agent may face, to ensure good performance---common practice in empirical studies. Unfortunately, safe values of these parameters are usually problem dependent. For example, in off-policy learning (e.g., learning from demonstrations), large importance sampling ratios can destabilize provably convergent gradient temporal difference learning methods, when the parameters are not set in a very particular way ($\lambda=0$) \cite{white2015thesis}. In such situations, we turn to meta-learning algorithms that can adapt the parameters of the agent continuously, based on the stream of experience and some notion of the agent's own learning progress. These meta-learning approaches can potentially improve robustness, and also help the agent specialize to the current task, and thus improve learning speed.


Temporal difference learning methods make use of two important parameters: the step-size parameter and the trace-decay parameter. 
The step-size parameter is the same as those used in stochastic gradient descent, and there are algorithms available for adjusting this parameter online, in reinforcement learning \cite{dabney2012adaptive}. For the trace decay parameter, on the other hand, we have no generally applicable meta-learning algorithms that are compatible with function approximation, incremental processing, and off-policy sampling. 

      

The difficulty in adapting the trace decay parameter, $\lambda$, mainly arises from the fact
that it has seemingly multiple roles and also influences the fixed-point solution. 
This parameter was introduced in Samuel's checker player \cite{samuel1959some},
and later described as interpolation parameter between offline TD(0) and Monte-Carlo sampling (TD($\lambda=1$) by Sutton\citet{sutton1988learning}.
It has been empirically demonstrated that values of $\lambda$ between zero and one often perform the best in practice \cite{sutton1988learning,sutton1998introduction,vanseijen2014true}. This trace parameter can also be viewed as a
bias-variance trade-off parameter: $\lambda$ closer to one is less biased
but likely to have higher variance, where $\lambda$ closer to zero is more biased,
but likely has lower variance.  
However, it has also been described as a credit-assignment parameter \cite{singh1996reinforcement},
as a method to encode probability of transitions \cite{sutton1994onstep}, a way to incorporate the agent's confidence in its value function estimates \cite{sutton1998introduction, tesauro1992practical}, 
and as an averaging of n-step returns \cite{sutton1998introduction}.
Selecting $\lambda$ is further complicated
by the fact that $\lambda$ is
a part of the problem definition:
the solution to the Bellman fixed point equation is dependent on the choice of $\lambda$ (unlike the step-size parameter). 


There are few approaches for setting $\lambda$, and most existing work is limited to special cases.
For instance, several approaches have analyzed setting $\lambda$ for variants of TD that were
introduced to simplify the analysis, 
including phased TD \cite{kearns2000bias}
and TD$^*(\lambda)$ \cite{schapire1996ontheworst}. 
Though both provide valuable insights into the role of $\lambda$, the analysis does not easily extend
to conventional TD algorithms. 
 Sutton and Singh \cite{sutton1994onstep} investigated tuning both the learning rate parameter and $\lambda$,  
and proposed two meta-learning algorithms. The first assumes the problem can be modeled by an acyclic MDP, and the other requires access to the transition model of the MDP. Singh and Dayan \cite{singh1996analytical} and Kearns and Singh \cite{kearns2000bias} contributed extensive simulation studies of the interaction between $\lambda$ and other agent parameters on a chain MDP, but again relied on access to the model and offline computation.
The most recent study \cite{downey2010temporal} 
explores a Bayesian variant of TD learning, but requires a batch of samples and can only be used off-line. 
Finally, Konidaris et al.~\citet{konidaris2011td} introduce TD$_\gamma$ as a method to remove the $\lambda$ parameter altogether.
Their approach, however, has not been extended to the off-policy setting and their full algorithm is too computationally expensive for incremental estimation, while their incremental variant
introduces a sensitive meta-parameter.
Although this long-history of prior work has helped develop our intuitions about $\lambda$, the available solutions are still far from the use cases outlined above.

This paper introduces an new objective based on locally optimizing bias-variance, 
which we use to develop an efficient, incremental algorithm for learning
state-based $\lambda$.
We use a forward-backward analysis \cite{sutton1998introduction} to derive an incremental algorithm
to estimate the variance of the return.
Using this estimate, we obtain a closed-form estimate of $\lambda$ on each time-step.
Finally, we empirically demonstrate the generality of the approach 
with a suite of on-policy and off-policy experiments. Our results show that 
our new algorithm, \lamname, is
consistently amongst the best performing, adapting as the problem changes,
whereas any fixed approach works well in some settings and poorly in anothers. 

%

\section{Background}
We model the agent's interaction with an unknown environment as a discrete time Markov Decision Process (MDP). A MDP is characterized by a finite set of states $\Ss$, set of actions $\As$, a reward function $r: \Ss \times \Ss \rightarrow \Rn$, and generalized state-based discount $\gamma: \Ss \in [0,1]$, which encodes the level of discounting
per-state (e.g., a common setting is a constant discount for all states). On each of a discrete number of timesteps, $t=1,2,3,\ldots$, 
the agent observes the current state $S_t$, selects an action $A_t$, according to its target policy $\pi: \Ss\times\As\rightarrow[0,1]$, and the environment transitions to a new state $S_{t+1}$ and emits a reward $R_{t+1}$. The state transitions are governed by the transition function $P:\Ss\times\As\times\Ss\rightarrow[0,1]$, where $P(S_t,A_t,S_{t+1})$ denotes the probability of transitioning from $S_t$ to $S_{t+1}$, due to action $A_t$.  At timestep $t$, the future rewards are summarized by the Monte Carlo (MC) return $\G_t\in\Rn$ defined by the infinite discounted sum
\begin{align*}
\G_t\ &\defeq R_{t+1} + \gamma_{t+1} R_{t+2} + \gamma_{t+1}\gamma_{t+2} R_{t+3} + \ldots && \triangleright \gamma_t = \gamma(S_t)\\
&= R_{t+1} + \gamma_{t+1} \G_{t+1}
.
\end{align*} 

The agent's objective is to estimate the expected return or value function, $v^\pi: \Ss\rightarrow\Rn$, defined as $v^{\pi}(s) \defeq \E[\G_t | S_t = s, A_t \sim \pi]$.
We estimate the value function using the standard framework of linear function approximation. We assume the state of the environment at time $t$ can be characterized by a fixed-length feature vector $\xvec_t\in\Rn^n$, where $n\ll |\Ss|$; implicitly, $\xvec_t$ is a function of the random variable $S_t$. The agent uses a linear estimate of the value of $S_t$: the inner product of $\xvec_t$ and a modifiable set of weights $\w\in\Rn^n$,
$\hat{v}(S_t,\w) \defeq \xvec_t^\top\w$, with mean-squared error (MSE) $=\sum_{s \in \Ss} d(s) (v(s) - \hat{v}(s,\w))^2$, where $d:\Ss \rightarrow[0,1]$ encodes the distribution over states induced by the agent's behavior in the MDP. 

\newcommand{\Glambda}{G^{\lambda}}

Instead of estimating the expected value of $\G_t$, we can estimate
a $\lambda$-return that is expected to have lower variance
\begin{align*}
\Glambda_t \defeq R_{t+1} + \gamma_{t+1} [ (1-\lambda_{t+1}) \xvec_{t+1}^\top\w + \lambda_{t+1} \Glambda_{t+1} ]
, 
\end{align*}
where the {\em trace decay function} $\lambda: \Ss \rightarrow [0,1]$ specifies the trace parameter as a function of state.
The trace parameter $\lambda_{t+1} = \lambda(s_{t+1})$ averages the estimate of the return, $\xvec_{t+1}^\top\w$,
and the $\lambda$-return starting on the next step, $\Glambda_{t+1}$. When $\lambda=1$, $\Glambda_t$ becomes the MC return $\G_t$, and the value function can be estimated by averaging rollouts from each state. When $\lambda=0$, $\Glambda_t$ becomes equal to the {\em one-step $\lambda$-return},
$R_{t+1} + \gamma_{t+1} \xvec_{t+1}^\top\w ,$
and the value function can be estimated by the linear TD(0) algorithm.
The $\lambda$-return when $\lambda \in (0,1)$ is often easier to estimate than MC, and yields more accurate predictions than using the one-step return.
 The intuition, is that the for large $\lambda$, the estimate is high-variance due to averaging possibly long trajectories of noisy rewards, but less bias because the initial biased estimates of the value function participate less in the computation of the return.
In the case of low $\lambda$, the estimate has lower-variance because fewer potentially noisy rewards participate in $\Glambda_t$, but there is more bias due to the increase role of the initial value function estimates.
We further discuss the intuition for this parameter in the next section.

The generalization to state-based $\gamma$ and $\lambda$ have not yet been widely
considered, though the concept was introduced more than a decade ago \cite{sutton1995td,sutton1999between}
and the generalization shown to be useful \cite{sutton1995td,maei2010gq,modayil2014multi,sutton2015anemphatic}.
The Bellman operator can be generalized to include state-based $\gamma$ and $\lambda$ 
(see \cite[Equation 29]{sutton2015anemphatic}),
where the choice of $\lambda$ per-state influences the fixed point.
Time-based $\lambda$, on the other hand, would not result in a well-defined fixed point. 
Therefore, to ensure a well-defined fixed point, we will design an objective
and algorithm to learn a state-based $\lambda$. 

This paper considers both on- and off-policy policy evaluation. In the more conventional on-policy learning setting, we estimate $v^\pi(s)$ based on samples generated while selecting actions according to the target policy $\pi$. In the off-policy case, we estimate $v^\pi(s)$ based on samples generated while selecting actions according to the behavior policy $\mu:\Ss\times\As\rightarrow [0,1]$, and $\pi \ne \mu$. In order to learn $\hat{v}$ in both these settings we use the GTD($\lambda$) algorithm \cite{maei2011gradient} specified by the following update equations:
\begin{align*}
&\rho_t \defeq \frac{\pi(S_t,A_t)}{\mu(S_t,A_t)}  \hspace{2.0cm}\triangleright \text{importance sampling ratio}\\
&\e_t \defeq \rho_t(\gamma_t\lambda_t \e_{t-1} + \xvec_t)   \hspace{1.9cm}\triangleright \text{eligibility trace}\\
&\delta_t \defeq R_{t+1} + \gamma_t \xvec_{t+1}^\top\w_t - \xvec_{t}^\top\w_t   \hspace{0.8cm}\triangleright \text{TD error}
\end{align*}
\begin{align*}
&\w_{t+1} \assigneq \w_t + \alpha( \delta_t\e_t - \gamma_t(1-\lambda_{t+1})\e_t^\top\h_t\xvec_{t+1})\\
&\h_{t+1} \assigneq \h_t + \alpha_\h( \delta_t\e_t - \xvec_t^\top\h_t)\xvec_t 
  \hspace{0.5cm}\triangleright \text{auxiliary weights}
\end{align*}
with step-sizes  $\alpha, \alpha_\h \in \Rn^+$ and an arbitrary initial $\w_{0}, \h_0$ (e.g., the zero vector).
The importance sampling ratio $\rho_t \in \Rn^+$ facilitates learning about rewards as if they were generated by following $\pi$, instead of $\mu$. This ratio can be very large if $\mu(S_t,A_t)$ is small, which can compound and destabilize learning.


\section{Objective for trace adaptation} 
 To obtain an objective for selecting $\lambda$,
 we need to clarify its role.
Although $\lambda$ was not introduced with the goal of trading off bias and variance \cite{sutton1988learning}, several algorithms and significant theory have developed its role as such \cite{kearns2000bias,schapire1996ontheworst}.
Other roles have been suggested;
however, as we discuss below, each of them can still be thought of as a bias-variance trade-off.

The $\lambda$ parameter has been described as a credit assignment parameter, which allows TD($\lambda$) to perform multi-step updates on each time step.
On each update, $\lambda_t$ controls the amount of credit assigned
to previous transitions, using the eligibility trace $\e$. 
For $\lambda_t$ close to 1, TD($\lambda$) assigns more credit for the current reward to previous transitions, resulting in updates to many states along the current trajectory. 
Conversely, for $\lambda_t = 0$, the eligibility trace is cleared and no
credit is assigned back in time, performing a single-step TD(0) update. 
In fact, this intuition can still be thought of as a bias-variance trade-off.
In terms of credit assignment, we ideally always want to send maximal credit $\lambda=1$, but decayed by $\gamma$, for the current reward, which is also unbiased. In practice, however, this often leads to high variance, and thus we mitigate the variance by choosing $\lambda$ less than one and speed learning overall, but introduce bias.

Another interpretation is that $\lambda$ should be set to reflect confidence in value function estimates \cite{tesauro1992practical, sutton1998introduction}.
If your confidence in the value estimate of state $s$ is high, then $\lambda(s)$ should be close to 0, meaning we trust the estimates provided by $\hat{v}$. 
If your confidence is low, suspecting that $\hat{v}$ may be inaccurate, then $\lambda(s)$ should be close to 1, meaning we trust observed rewards more. For example in states that are indistinguishable with function approximation (i.e., aliased states), we should not trust the $\hat{v}$ as much.
This intuition similarly translates to bias-variance.
If $\hat{v}$ is accurate, then decreasing $\lambda(s)$ does not incur (much) bias,
but can significantly decrease the variance since $\hat{v}$ gives the correct value.
If $\hat{v}$ is inaccurate, then the increased bias is not worth the reduced variance, so $\lambda(s)$ should be closer to 1
to use actual (potentially high-variance) samples. 

Finally, a less commonly discussed interpretation is that $\lambda$ acts as parameter
that simulates a form of experience replay (or model-based simulation of trajectories). One can imagine
that sending back information in eligibility traces is like simulating experience from a model,
where the model could be a set of trajectories, as in experience replay \cite{lin1992self}.
If $\lambda = 1$, the traces are longer and each update gets more trajectory information,
or experience replay.
If a trajectory from a point, however, was unlikely (e.g., a rare transition), 
we may not want to use that information.
Such an approach was taken by Sutton and Singh \citet{sutton1994onstep}, where $\lambda$ was set to the transition probabilities.
Even in this model-based interpretation, the goal in setting $\lambda$ becomes one 
of mitigating variance, without incurring too much bias.

Optimizing this bias-variance trade-off, however, is difficult because $\lambda$ 
affects the return we are approximating. Jointly optimizing for $\lambda$ across
all time-steps is generally not feasible. One strategy 
is to take a batch approach, where the optimal $\lambda$ is determined after seeing all the data \cite{downey2010temporal}.
Our goal, however, is to develop approaches for the online setting, 
where future states, actions, rewards and the influence of $\lambda$ have yet to be observed.
 
We propose to take a greedy approach: on each time step select $\lambda_{t+1}$ to optimize
the bias-variance trade-off \textit{for only this step}. 
This greedy objective corresponds to 
minimizing the mean-squared error between the unbiased $\lambda=1$ return
$\Gunbiased{t}$ and the estimate $\Gbiased{t}$
with $\lambda_{t+1} \in [0,1]$ with $\lambda = 1$ into the future after $t+1$ 
\begin{align*}
\Gbiased{t} \defeq \rho_t (R_{t+1} + \gamma_{t+1}[\oneminuslam{t+1}\x_{t+1}^\top\w_t+\lambda_{t+1} \underbrace{\Gunbiased{t+1}}_{\makebox[0pt]{\text{\tiny Monte Carlo}}}]) 
.
\end{align*}
Notice that $\Gbiased{t}$ interpolates between the current value estimate
and the unbiased $\lambda=1$ MC return, and so is not recursive. 
Picking $\lambda_{t+1} = 1$ gives an unbiased estimate, since then we would be estimating
$\Gbiased{t}  = \Gunbiased{t}$. 
We greedily decide how $\lambda_{t+1}$ should be set on this step to locally optimize the mean-squared error (i.e., bias-variance).
This greedy decision is made given both $\x_t$ and $\x_{t+1}$, which are both available
when choosing $\lambda_{t+1}$. To simplify notation in this section,
we assume that $\x_t$ and $\x_{t+1}$ are both given in the below expectations.  

To minimize the mean-squared error in terms of $\lambda_{t+1}$
\begin{align*}
\text{MSE}(\lambda_{t+1}) &\defeq \E\left[(\Gbiased{t} - \E[\Gunbiased{t}])^2\right]
\end{align*}
we will consider the two terms that compose the mean-squared error: 
the squared bias term 
and the variance term. 
\begin{align*}
\text{Bias}(\lambda_{t+1}) &\defeq \E[\Gbiased{t}] - \E[\Gunbiased{t}] \\
\text{Variance}(\lambda_{t+1}) &\defeq \var[\Gbiased{t}]\\
\text{MSE}(\lambda_{t+1}) &= \text{Bias}(\lambda_{t+1})^2 + \text{Variance}(\lambda_{t+1})
.
\end{align*}
Let us begin by rewriting the bias.
Since we are given $\x_{t}$, $\rho_t$, $\x_{t+1}$ and $\gamma_{t+1}$
when choosing $\lambda_{t+1}$,
\begin{align*}
\E[\Gbiased{t}]
&=
\rho_t \E\left[R_{t+1} + \gamma_{t+1}\oneminuslam{t+1}\x_{t+1}^\top\w_t+\lambda_{t+1} \Gunbiased{t+1}\right] \\
&=
\rho_t \E[R_{t+1}] + \rho_t\gamma_{t+1}\left(\oneminuslam{t+1}\x_{t+1}^\top\w_t+\lambda_{t+1} \E[\Gunbiased{t+1}] \right)
\end{align*}
For convenience, define 
\begin{align}
\text{err}(\w, \x_{t+1}) \defeq \E[\Gunbiased{t+1}] - \x_{t+1}^\top\w
\label{eq_err}
\end{align}
as the difference between the $\lambda=1$ return 
and the current approximate value from state $\x_{t+1}$ using weights $\w_t$. Using this definition, we can rewrite
\begin{align*}
&\oneminuslam{t+1}\x_{t+1}^\top\w_t+\lambda_{t+1} \E[\Gunbiased{t+1}]\\
&=
 \oneminuslam{t+1}(\E[\Gunbiased{t+1}]
-\text{err}(\w_t, \x_{t+1}) )+\lambda_{t+1} \E[\Gunbiased{t+1}] \\
&=
\E[\Gunbiased{t+1}]
-  \oneminuslam{t+1}\text{err}(\w_t, \x_{t+1}) 
\end{align*}
giving
\begin{align*}
\E[\Gbiased{t}]
&=
\rho_t \left(\E[R_{t+1}] + \gamma_{t+1}\E[\Gunbiased{t+1}] \right) \\
&\hspace{1.0cm}- \rho_t\gamma_{t+1}\oneminuslam{t+1}\text{err}(\w_t, \x_{t+1}) \\
&=
\E[\Gunbiased{t}] - \rho_t\gamma_{t+1}\oneminuslam{t+1}\text{err}(\w_t, \x_{t+1})
\end{align*}
\begin{align*}
\implies
\text{Bias}^2(\lambda_{t+1}) &= \left(\E[\Gunbiased{t}] - \E[\Gbiased{t}]\right)^2 \\
&=\rho_t^2 \gamma^2_{t+1}\oneminuslam{t+1}^2\text{err}^2(\w_t, \x_{t+1})
.
\end{align*}
For the variance term, we will assume that the noise in the reward $R_{t+1}$
given $\x_t$ and $\x_{t+1}$ is independent of the other dynamics \cite{mannor2004bias}, with variance $\sigma_r(\x_t, \x_{t+1})$.  
Again since we are
given $\x_{t}$, $\rho_t$, $\x_{t+1}$ and $\gamma_{t+1}$
\begin{align*}
&\var[\Gbiased{t}]\\
&= \rho_t^2 \var[R_{t+1} +  \underbrace{\gamma_{t+1}\oneminuslam{t+1} \x_{t+1}^\top\w_t}_{\text{constant given $\rho_t, \x_{t+1}, \gamma_{t+1}$}} + \gamma_{t+1}\lambda_{t+1} \Gunbiased{t+1} ]\\
&= \rho_t^2 \underbrace{\var[R_{t+1} ] + \rho_t^2 \gamma_{t+1}^2\lambda_{t+1}^2\var[ \Gunbiased{t+1} ]}_{\text{independent given $\x_t, \x_{t+1}$}}\\
&= \rho_t^2 \sigma_r(\x_t,\x_{t+1}) + \rho_t^2\gamma_{t+1}^2\lambda_{t+1}^2\var[\Gunbiased{t+1}] \\
&\implies 
\text{Variance}(\lambda_{t+1}) = \rho_t^2\gamma_{t+1}^2\lambda_{t+1}^2\var[\Gunbiased{t+1}] + \rho_t^2 \sigma_r(\x_t,\x_{t+1})
\end{align*}
Finally, we can drop the constant  $\rho_t^2 \sigma_r(\x_t,\x_{t+1})$ in the objective, and
drop the $\rho_t^2\gamma_{t+1}^2$ in both the bias and variance terms as it only scales the objective, 
giving the optimization
\begin{align*}
&\min_{\lambda_{t+1} \in [0,1]} \text{Bias}^2(\lambda_{t+1}) + \text{Variance}(\lambda_{t+1}) \\
\equiv
&\min_{\lambda_{t+1} \in [0,1]} \oneminuslam{t+1}^2\text{err}^2_{t+1}(\w_t) + \lambda_{t+1}^2 \var[\Gunbiased{t+1}]
.
\end{align*}
We can take the gradient of this optimization to find a closed form solution
\begin{align}
&-2\oneminuslam{t+1}\text{err}^2_{t+1}(\w_t) + 2\lambda_{t+1} \var[\Gunbiased{t+1}] = 0 \nonumber\\
&\implies \left(\var[\Gunbiased{t+1}] + \text{err}^2_{t+1}(\w_t) \right) \lambda_{t+1}
- \ \text{err}^2_{t+1}(\w_t) = 0 \nonumber\\
&\implies 
 \lambda_{t+1} = \frac{ \text{err}^2_{t+1}(\w_t)}{\var[\Gunbiased{t+1}] + \text{err}^2_{t+1}(\w_t)} \label{eq_lambda}
\end{align}
which is always feasible, unless both the variance and error are zero (in which case, any choice
of $\lambda_{t+1}$ is equivalent). 
Though the importance sampling ratio $\rho_t$ does not affect the choice
of $\lambda$ on the current time step, it can have a dramatic effect on $\var[\Gunbiased{t+1}]$ into the future via the eligibility trace. 
For example, when the target and behavior policy are strongly mis-matched, $\rho_t$ can be large, which
multiplies into the eligibility trace $\e_t$. If several steps have large $\rho_t$,
then $\e_t$ can get very large. In this case, the equation in \eqref{eq_lambda} would select a small $\lambda_{t+1}$,
significantly decreasing variance.


\section{Trace adaptation algorithm}

To approximate the solution to our proposed optimization, we need a way to approximate
the error and the variance terms in equation \eqref{eq_lambda}.
To estimate the error, we need an estimate of the expected return from each state, 
$\E[\Gunbiased{t}]$. 
To estimate the variance, we need to obtain an estimate of $\E[\Gunbiased{t}^2]$, and
then can use $\var[\Gunbiased{t}] = \E[\Gunbiased{t}^2] - \E[\Gunbiased{t}]^2$.
The estimation of the expected return is in fact the problem
tackled by this paper, and one could use a TD algorithm,
learning weight vector $\herr$ to obtain approximation $\x_t^\top \herr$ to $\E[\Gunbiased{t}]$. 
This approach may seem problematic, as this sub-step appears to
be solving the same problem we originally aimed to solve. 
However, as in many meta-parameter optimization approaches,
this approximation can be inaccurate and still adequately guide selection
of $\lambda$. We discuss this further in the experimental results section. 


\begin{algorithm}[h]
\caption{Policy evaluation with \lamname}\label{alg_tderror}
\begin{algorithmic}[]
     \STATE  $\w_t \gets \zerovec$  // Main weights
     \STATE  $\h_t \gets \zerovec$ //  Auxiliary weights for off-policy learning
     \STATE  $\e_t \gets \zerovec$ // Main trace parameters
  \STATE  $\x_t \gets$ the initial observation  
     \STATE $ \herr \gets \tfrac{\text{Rmax}}{1-\gamma} \times \onevec, \hvar \gets 0.0 \times \onevec$  // Aux. weights for $\lambda$ 
     \STATE  $\gtrace{t} \gets \zerovec, \gtwotrace{t} \gets \zerovec$ // Aux. traces 
\REPEAT
\STATE Take action accord. to $\pi$, observe $\x_{t+1}$, reward $r_{t+1}$
\STATE $\rho_t = \pi(s_t,a_t) / \mu(s_t,a_t)$ \ \ \ // In on-policy, $\rho_t = 1$
\STATE $\lambda_{t+1} \gets \text{\lamname}(\herr, \hvar, \w_t, \x_t, \x_{t+1},r_{t+1}, \rho_t)$
\STATE // Now can use any algorithm, e.g., GTD
\STATE $\delta_t = r_{t+1} + \gamma_{t+1} w_t^\top\x_{t+1}  - \w_t^\top\x_t$
\STATE $\e_t = \rho_t(\gamma_t\lambda_t \e_{t-1}+ \x_t)$
\STATE $\w_{t+1} = \w_t + \alpha ( \delta_t\e_t - \gamma_{t+1} (1- \lambda_{t+1}) \x_{t+1})(\e^\top_t \h_t))$
\STATE $\h_{t+1} = \h_t + \alpha_\h (\delta_t\e_t- (\h^\top_t\x_t)\x_t)$
\UNTIL{agent done interaction with environment}
\end{algorithmic}
\end{algorithm}


\begin{algorithm}[h]
\caption{\lamname($\herr, \hvar, \w_t, \x_t, \x_{t+1},r_{t+1}, \rho_t)$}\label{alg_traceadapt}
\begin{algorithmic}[]
\STATE // Use GTD to update $\herr$
\STATE $\bar{g}_{t+1} \gets \x_{t+1}^\top \herr$
\STATE $\delta_t \gets r_{t+1} + \gamma_{t+1} \bar{g}_{t+1}  - \x_t^\top \herr$
\STATE $\gtrace{t} = \rho_t(\gamma_t\gtrace{t-1} + \x_t)$
\STATE $\herr = \herr + \alpha \delta_t \gtrace{t} $
\STATE // Use VTD to update $\hvar$
\STATE $\rbar_{t+1} \gets \rho_t^2 r_{t+1}^2 + 2\rho_t^2 \gamma_{t+1} r_{t+1} \bar{g}_{t+1}$
\STATE $\gambar_{t+1} \gets \rho_t^2\gamma_{t+1}^2$
\STATE $\bar{\delta}_t \gets \rbar_{t+1} + \gambar_{t+1} \x_{t+1}^\top \hsq  - \x_t^\top \hsq$
\STATE $\gtwotrace{t} = \gambar_{t}\gtwotrace{t-1} + \x_{t}$
\STATE $\hvar = \hvar + \alpha \bar{\delta}_t \ztrace_t$
\STATE // Compute $\lambda$ estimate
\STATE $\text{errsq} = (\bar{g}_{t+1} - \x_{t+1}^\top \w_t)^2$
\STATE $\text{varg} = \max(0,   \x_{t+1}^\top\hvar - (\bar{g}_{t+1})^2)$
\STATE $\lambda_{t+1} = \text{errsq}  / (\text{varg}+ \text{errsq} )$
\RETURN $\lambda_{t+1}$
\end{algorithmic}
\end{algorithm}

Similarly, we would like to estimate $\E[\Gunbiased{t}^2]$ with $\hvar \x_t^\top$
by learning $\hvar$; 
estimating the variance or the second moment of the return, however, has not been extensively studied. 
Sobel \citet{sobel1982thevariance}
provides a Bellman equation for the variance of the $\lambda$-return, when $\lambda=0$.
There is also an extensive literature on risk-averse MDP learning,
where the variance of the return is often used as a measure \cite{morimura2010parametric,mannor2011mean,prashanth2013actor,tamar2013temporal};
however, an explicit way to estimate the variance of the return for $\lambda > 0$
is not given. There has also been some work on estimating the variance of the {\em value function} \cite{mannor2004bias,white2010interval},
for general $\lambda$; though related, this is different than estimating the variance of the {\em $\lambda$-return}. 

In the next section, we provide a derivation for a new algorithm called variance temporal difference learning (VTD), to approximate the second moment of the return for any state-based $\lambda$. The general VTD updates are given at the end of Section \ref{sec_vtd}.
For \lamname, we use VTD to estimate the variance, with the complete algorithm
summarized in Algorithm \ref{alg_tderror}. 
We opt for simple meta-parameter settings, so that no
additional parameters are introduced.
We use the same step-size $\alpha$ that is used for the main weights
to update $\herr$ and $\hvar$.
In addition, we set the weights $\herr$ and $\hvar$
 to reflect \apriori\ estimates of error and variance.  
 As a reasonable rule-of-thumb,
 $\herr$ should be set larger than $\hvar$,
 to reflect that initial value estimates are inaccurate.
 This results in an estimate of 
 variance $\var[\G_{t+1}] \approx \x_{t+1}^\top \hvar - (\x_{t+1}^\top \herr)^2$
 that is capped at zero until $\x_{t+1}^\top \hvar$ becomes larger
 than $(\x_{t+1}^\top \herr)^2$. 

\section{Approximating the second moment of the return}

In this section, we derive the general VTD algorithm to
approximate the second moment of the $\lambda$-return. 
Though we will set $\lambda_{t+1} = 1$ in our algorithm,
we nonetheless provide the more general algorithm
as the only model-free variance estimation approach
for general $\lambda$-returns. 

The key novelty is in determining a Bellman operator for the
squared return, which then defines a fixed-point objective,
called the \varMSPBE. 
With this Bellman operator and recursive form for
the squared return, we derive a gradient TD algorithm, called VTD, for estimating the second moment.
To avoid confusion with parameters for the main algorithm, as a general rule throughout the document, 
the additional parameters used to estimate the second moment have a bar.
For example, $\gamma_{t+1}$ is the discount for the main problem,
and $\gambar_{t+1}$ is the discount for the second moment.

\subsection{Bellman operator for squared return}

The recursive form for the squared-return is
\begin{align*}
\glambdasq{t} &= \rho_t^2 (\gbar_{t}^2 + 2 \gamma_{t+1} \lambda_{t+1} \gbar_{t} \glambda{t+1} + \gamma_{t+1}^2 \lambda_{t+1}^2 \glambdasq{t+1})\\
&= \rbar_{t+1} + \gambar_{t+1}  \glambdasq{t+1}
\end{align*}
where 
for a given $\lambda: \Ss \rightarrow [0,1]$ and $\w$, 
\begin{align*}
\gbar_{t} &\defeq R_{t+1} + \gamma_{t+1} (1-\lambda_{t+1}) \x_{t+1}^\top \w \\
\rbar_{t+1} &\defeq \rho_t^2\gbar_{t}^2 + 2 \rho_t^2\gamma_{t+1} \lambda_{t+1} \gbar_{t} \glambda{t+1}\\
\gambar &\defeq \rho_t^2 \gamma_{t+1}^2 \lambda_{t+1}^2
.
\end{align*}
The $\w$ are the weights for the $\lambda$-return, 
and not the weights $\hsq$ we will learn
for approximating the second moment. 
For further generality, we introduce 
a meta-parameter $\lambar_t$
%
\begin{align*}
\glambdasqbar_{t} &\defeq \rbar_{t+1} + \gambar_{t+1} \left( (1-\lambar_{t+1}) \xvec_{t+1}^\top \hvar + \lambar_{t+1} \glambdasqbar_{t+1}\right)
\end{align*}
to get a $\lambar$-squared-return
where for $\lambar_{t+1} = 1$, $\glambdasqbar_{t} = \glambdasq{t}$.
This meta-parameter $\lambar$ plays the same role for estimating $\glambdasq{t}$
as $\lambda$ for estimating $G_{t}$. 

%
We can define a generalized Bellman operator $\Bellmansq$
for the squared-return, using this above recursive form.
The goal is to obtain the fixed point $\Bellmansq \Vsq = \Vsq$,
where a fixed point exists
\textit{if the operator is a contraction}.
For the first moment, the Bellman operator is known
to be a contraction \cite{tsitsiklis1997ananalysis}.
This result, however, does not immediately extend here because,
thought $\rbar_{t+1}$ is a valid finite reward, 
$\gambar_{t+1}$ does not satisfy $\gambar_{t+1} \le 1$,
because $\rho_t^2$ can be large.

We can nonetheless define such a Bellman operator 
for the $\lambar$-squared-return and determine if a fixed point exists. 
Interestingly, $\gambar_{t+1}$ can in fact
be larger than $1$, and we can still obtain a contraction.
To define the Bellman operator, we use a 
recent generalization that
enables the discount to be defined as a function of $(s,a,s')$ \cite{white2016transition},
rather than just as a function of $s'$.
We first define $\Vsq$, the expected $\lambar$-squared-return 
%
$$\Vsq \defeq \sum_{t=0}^\infty (\Psq)^t \rbarvec,$$ where
\begin{align}\Psq(s,s') &\defeq \sum_{s,a,s'} P(s,a,s') \mu(s,a) \rho(s,a)^2 \gamma(s')^2 \lambda(s')^2 \label{eq_bellman}\\
\rbarvec &\defeq \sum_{s,a,s'} P(s,a,s') \mu(s,a) \rbar(s,s') \nonumber
.
\end{align}
Using similar equations to the generalized Bellman operator \cite{white2016transition},
we can define 
$$\Bellmansq \Vsq = \sum_{t=0}^\infty (\Psq \Lambdamat)^t \left(\rbarvec + \Psq (\eye - \bar{\Lambdamat}) \Vsq \right)
$$
where $\bar{\Lambdamat} \in \RR^{|\Ss| \times |\Ss|}$ is a matrix with $\lambar(s)$ on the diagonal, for all $s\in \Ss$. 
The infinite sum is convergent if the maximum singular value of $\Psq \bar{\Lambdamat}$
is less than 1, giving solution $\sum_{t=0}^\infty (\Psq \bar{\Lambdamat})^t  = (\eye - \Psq \bar{\Lambdamat})^{-1}$. 
Otherwise, however, the value is infinite and one can see that
in fact the variance of the return is infinite! 

We can naturally investigate when the second moment of the return is guaranteed to be finite. 
This condition on $\Psq \bar{\Lambdamat}$ should facilitate identifying
theoretical conditions on the target and behavior policies that enable finite variance of the return. 
This theoretical characterization is outside of the scope of this work,
but we can reason about different settings that provide a well-defined, finite fixed point. 
First, clearly setting $\lambda_{t+1} = 0$ for every state ensures
a finite second moment, given a finite $\Vsq$, regardless of policy mis-match.
For the on-policy setting, where $\rho_t = 1$,
$\gambar_{t+1} \le \gamma_{t+1}$ and so a well-defined fixed point
exists, under standard assumptions (see \cite{white2016transition}). 
For the off-policy setting, 
if $\gambar_{t+1} = \lambda_{t+1}^2 \gamma_{t+1}^2 \rho_t^2 < 1$,
this is similarly the case. 
Otherwise, a solution may still exist, by ensuring that the maximum singular
value of $\Psq$ is less than one; we hypothesize that 
this property is unlikely if there is a large mis-match
between the target and behavior policy, causing many large $\rho_t$.
An important future avenue is to understand the required similarity between $\pi$ and $\mu$
to enable finite variance of the return, for any given $\lambda$.
Interestingly, the \lamname\ algorithm should adapt to such infinite variance
settings, where \eqref{eq_lambda} will set $\lambda_{t+1} = 0$.

\subsection{VTD derivation}\label{sec_vtd}

\newcommand{\Pimat}{\boldsymbol{\Pi}}

In this section, we propose \varMSPBE, 
the mean-squared projected Bellman error (MSPBE) objective for the $\lambar$-squared-return,
and derive VTD to optimize this objective.
Given the definition of the generalized Bellman operator $\Bellmansq$,
the derivation parallels GTD($\lambda$) for the first moment \cite{maei2011gradient}.
The main difference is in obtaining unbiased estimates of
parts of the objective;
we will therefore focus the results on this novel aspect,
summarized in the below two theorems and corollary. 

Define the error of the estimate $\x_t^\top \hsq$ to the future $\lambar$-squared-return
\begin{align*}
\deltalambar{t} \defeq \glambdasqbar_{t} -  \x_t^\top \hsq
\end{align*}
%
%
and, as in previous work \cite{sutton2009fast,maei2011gradient}, we define the MSPBE that corresponds to $\Bellmansq$
\begin{align*}
\varMSPBE(\hsq) 
&\defeq \E[\deltalambar{t} \x_t]^\top \E[\x_t \x_t]^{-1} \E[\deltalambar{t} \x_t]
.
\end{align*}
To obtain the gradient of the objective,
we prove that we can obtain an unbiased
sample of $\deltalambar{t} \x_t$ (a forward view) using 
a trace of the past (a backward view). 
The equivalence is simpler if we assume that we have access to an 
estimate of the first moment of the $\lambda$-return. For our setting,
we do in fact have such an estimate, because we simultaneously learn $\herr$.
We include the more general expectation equivalence
in Theorem \ref{thm_mse}, with all proofs in the appendix. 
%
%
\begin{theorem}\label{thm_main}
For a given 
unbiased estimate $\bar{g}_{t+1}$ of \\
$\E[\glambda{t+1} | S_{t+1}] $,
define
\begin{align*}
\deltabar_t &\defeq  (\rho_t^2\gbar_{t}^2 + 2 \rho_t^2\gamma_{t+1} \lambda_{t+1} \gbar_{t} \bar{g}_{t+1}) +  \gambar_{t+1} \xvec_{t+1}^\top \hsq_t - \xvec_{t}^\top \hsq_t\\
\gtwotrace{t} &\defeq \x_t+ \gambar_{t+1}\lambar_{t+1}\gtwotrace{t-1}
\end{align*}
Then
\vspace{-0.6cm}
\begin{align*}
\E[\deltalambar{t} \x_t] &= \E[\deltabar_{t} \ztrace_{t}]
\end{align*}
%
\end{theorem} 
\begin{theorem}\label{thm_mse}
\begin{align*}
\E[\rbar_{t+1}  \ztrace_t] = E[\rho_{t+1}^2\gbar_t^2 \gtwotrace{t} ] + 2 E[\rho_{t+1}^2\gbar_t (\gvectracer{t} + \gvectracex{t} \w_t)] 
\end{align*}
where
\begin{align*}
\gvectracer{t} &\defeq \rho_t \gamma_{t} \lambda_{t} (R_{t} \gtwotrace{t-1}  + \gvectracer{t-1})\\
\gvectracex{t} &\defeq \rho_t \gamma_{t} \lambda_{t} (\gamma_{t} (1-\lambda_{t}) \gtwotrace{t-1} \x_{t}^\top   + \gvectracex{t-1})
\end{align*}
\end{theorem} 
\begin{corollary}\label{cor_lamvtd}
For $\lambda_t = 1$ for all $t$, 
\begin{align*}
\E[\rbar_{t+1}  \ztrace_t] = E[\rho_t^2 R_{t+1}^2 \gtwotrace{t} ] + 2 E[R_{t+1} \gvectracer{t} ]
\end{align*}
where
\vspace{-0.2cm}
\begin{align*}
\gtwotrace{t} &=\x_t+ \rho_{t-1}^2\gamma_{t}^2 \gtwotrace{t-1}\\
\gvectracer{t} &= \rho_t \gamma_{t} (R_{t} \gtwotrace{t-1}  + \gvectracer{t-1})
.
\end{align*}
\end{corollary} 
To derive the VTD algorithm, we take the gradient of the \varMSPBE.
As this again parallels GTD($\lambda$), we include the derivation in
the appendix for completeness and provide only the final result here. 
\begin{align*}
&-\tfrac{1}{2} \nabla \varMSPBE(\hsq) \\
&=  \E[\deltalambar{t} \x_t]  - \E[\gambar_{t+1} (1-\lambar_{t+1}) \xvec_{t+1} \gtwotrace{t}^\top]\E[\x_t \x_t]^{-1} \E[\deltalambar{t} \x_t]
.
\end{align*}
As with previous gradient TD algorithms, we will learn an
auxiliary set of weights $\hsq$ to 
estimate a part of this objective: $\E[\x_t \x_t]^{-1} \E[\deltalambar{t} \x_t]$. 
To obtain such an estimate, notice that $\auxhsq = \E[\x_t \x_t]^{-1} \E[\deltalambar{t}\x_t]$
corresponds to an LMS solution, where the goal is to obtain $\x_t^\top\auxhsq$
that estimates $\E[\deltalambar{t} | \xvec_t]$.
Therefore, 
we can use an LMS update
for $\auxhsq$, giving the final set of update equations for VTD: 
\begin{align*}
\bar{g}_{t+1} &\gets \x_{t+1}^\top \herr\\
\rbar_{t+1} &\gets  \rho_t^2\gbar_{t}^2 + 2 \rho_t^2\gamma_{t+1} \lambda_{t+1} \gbar_{t} \bar{g}_{t+1} \\
\gambar_{t+1} &\gets \rho_t^2 \gamma_{t+1}^2 \lambda_{t+1}^2\\
\deltabar_t &\gets  \rbar_{t+1} +  \gambar_{t+1} \xvec_{t+1}^\top \hsq_t - \xvec_{t}^\top \hsq_t\\ 
\hsq_{t+1} &\gets \hsq_t + \hsqauxstep \deltabar_t \gtwotrace{t} -  \hsqauxstep \gambar_{t+1} (1-\lambar_{t+1})\xvec_{t+1} \gtwotrace{t}^\top \auxhsq_t \\
\auxhsq_{t+1} &\gets \auxhsq_t + \hsqauxstep \deltabar_t \gtwotrace{t} -  \hsqauxstep\x_t^\top \auxhsq_t \x_t
.
\end{align*}
%
%
For \lamname, we set $\lambar_{t+1} = 1$, causing the term with the auxiliary weights to be multiplied by $1-\lambar_{t+1}$,
and so removing the need to approximate $\auxhsq$.


\section{Related work}

There has been a significant effort to empirically investigate $\lambda$,
typically using batch off-line computing and model-based techniques. 
 Sutton and Singh \citet{sutton1994onstep} investigated tuning both $\alpha$ and $\lambda$. They proposed three algorithms, the first two assume the underlying MDP has no cycles, and the third makes use of an estimate of the transition probabilities and is thus of most interest in tabular domains. Singh and Dayan \citet{singh1996analytical} provided analytical expression for bias and variance, given the model.
They suggest that there is a largest feasible step-size $\alpha$, below which bias converges to zero and variance converges to a non-zero value, and above which bias and/or variance may diverge. 
Downey and Sanner \citet{downey2010temporal} used a Bayesian variant of TD learning, requiring a batch of samples and off-line computation, but did provide an empirical demonstration off optimally setting $\lambda$ after obtaining all the samples. Kearns and Singh\citet{kearns2000bias} compute a bias-variance error bound for a modification of TD called phased TD. In each discrete phase the algorithm is given $n$ trajectories from each state. Because we have $n$ trajectories in each state the effective learning rate is $1/n$ removing the complexities of sample averaging in the conventional online TD-update. 
 The error bounds are useful for, among other things, computing a new $\lambda$ value for each phase which outperforms any fixed $\lambda$ value, empirically demonstrating the utility of changing $\lambda$. 

There has also been a significant effort to theoretically characterizing $\lambda$.
Most notably, the work of Schapire and Warmuth \cite{schapire1996ontheworst} contributed a finite sample analysis of incremental TD-style algorithms. They analyze a variant of TD called TD$^\star(\lambda)$, which although still linear and incremental,
computes value estimates quite differently. 
The resulting finite sample bound is particularly interesting, as it does not rely on model assumptions,
using only access to 
a sequence of feature vectors, rewards and returns. 
Unfortunately, the bound cannot be analytically minimized to produce an optimal $\lambda$ value. 
They did simulate their bound, further verifying the intuition that
$\lambda$ should be larger if the best linear predictor is inaccurate, small if accurate and an intermediate value
otherwise. 
Li \citet{li2008aworst} later derived similar bounds for another gradient descent algorithm, called residual gradient. This algorithm, however, does not utilize eligibility traces and converges to a different solution than TD methods when function approximation is used \cite{sutton2009fast}. 

Another approach involves removing the $\lambda$ parameter altogether, in an effort to improve robustness.
Konidaris et al. \citet{konidaris2011td} introduced a new TD method called TD$_\gamma$.
Their work defines a plausible set of assumptions implicitly made when constructing the $\lambda$-returns, and then relaxes one
of those assumptions. They derive an exact (but computationally expensive) algorithm, TD$_\gamma$, that no longer depends on a choice of $\lambda$
and performs well empirically in a variety of policy learning benchmarks. The incremental approximation to TD$_\gamma$ also performs reasonably well, but appears to be somewhat sensitive to the choice of
meta parameter $C$, and often requires large $C$ values to obtain good performance. This can be problematic,
as the complexity grows as $O(C \xdim)$, where $\xdim$ is the length of the trajectories---not linearly in the feature vector size. Nonetheless, TD$_\gamma$ constitutes a reasonable way to reduce parameter sensitivity in the on-policy setting.
Garcia and Serre\citet{garcia2001from} proposed a variant of Q-learning, for which the optimal value of $\lambda$ can be computed online. Their analysis, however, was restricted to the tabular case.
Finally, Mahmood et al. \citet{mahmood2014weighted} introduced weighted importance sampling
for off-policy learning; though indirect,
this is a strategy for enabling larger $\lambda$ to be selected, without destabilizing off-policy learning. 

This related work has helped shape our intuition on the role of $\lambda$,
and, in special cases, provided effective strategies for adapting $\lambda$.
In the next section, we add to existing work with
an empirical demonstration of \lamname, the first $\lambda$-adaptation algorithm
for off-policy, incremental learning,
developed from a well-defined, greedy objective. 
 




\section{Experiments}

\newcommand{\gwidth}{0.3\textwidth}
\newcommand{\addlabel}[1]{\parbox{0.5cm}{\vspace{-3.0cm}\rotatebox[origin=c]{90}{#1}}\hspace{-0.5cm}}

 \begin{figure*}
\raggedleft
\hspace*{-1.2cm}
\begin{tabular}{cccc}
\addlabel{Error}  & \includegraphics[width=\gwidth]{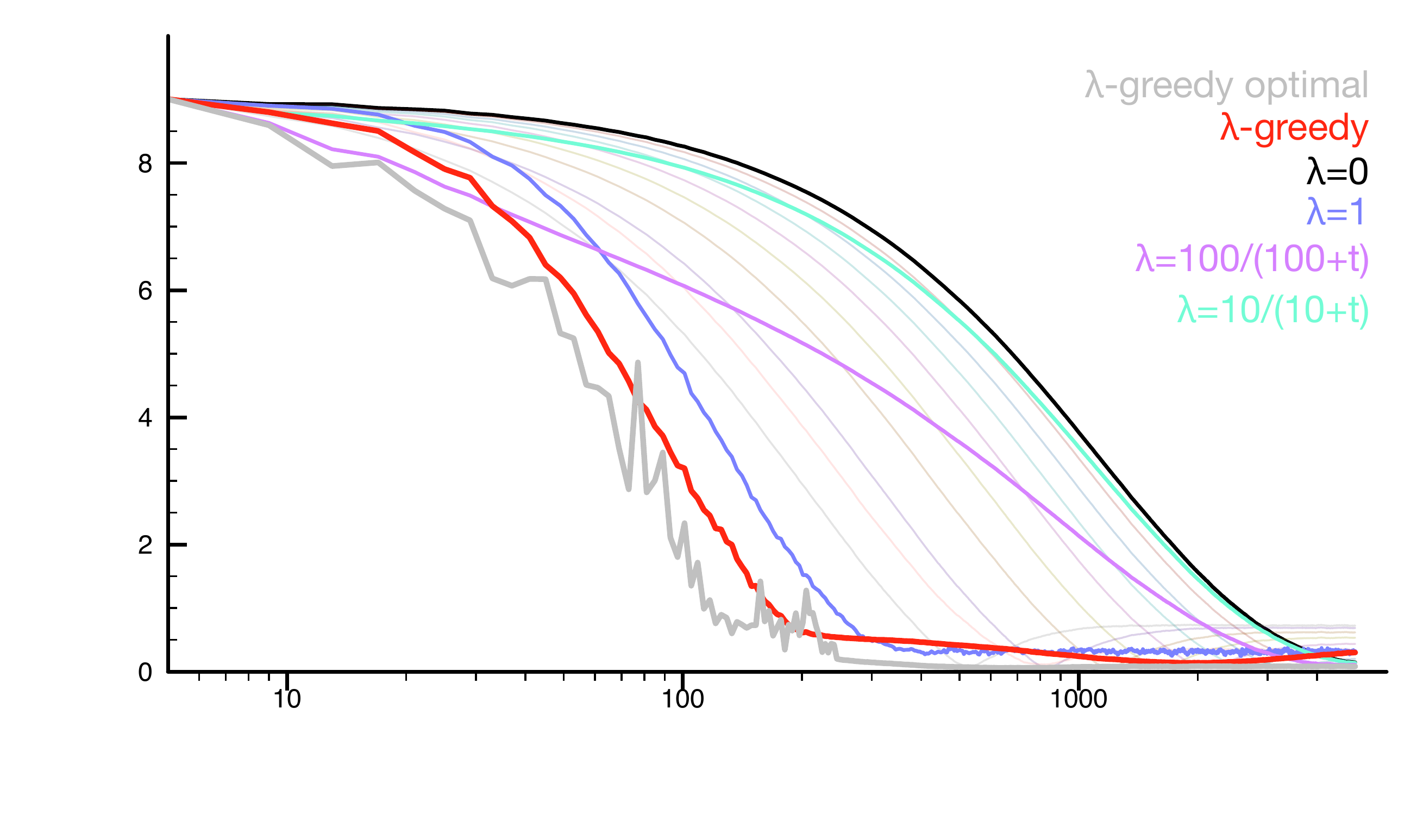}  & \includegraphics[width=\gwidth]{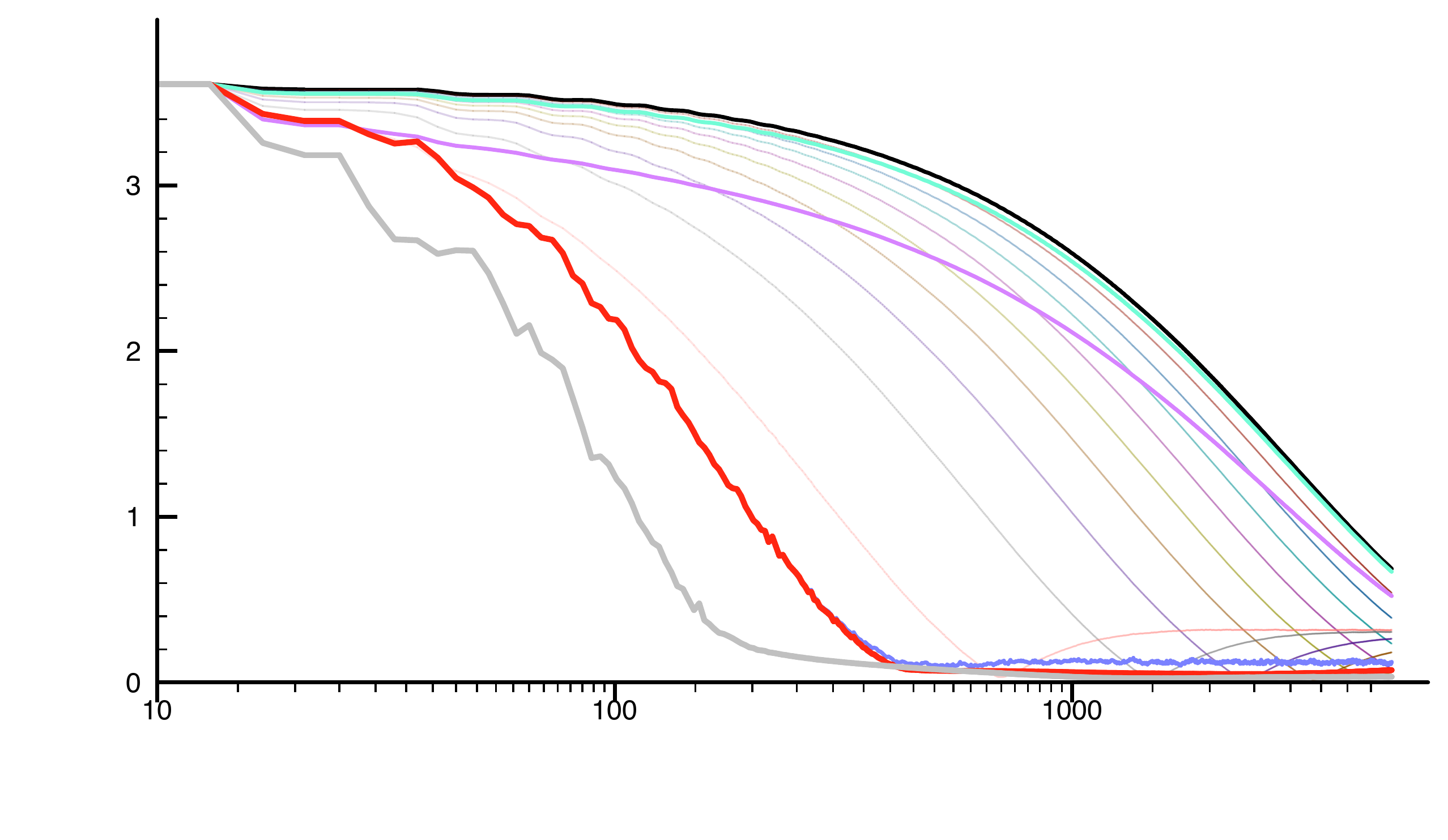} &    \includegraphics[width=\gwidth]{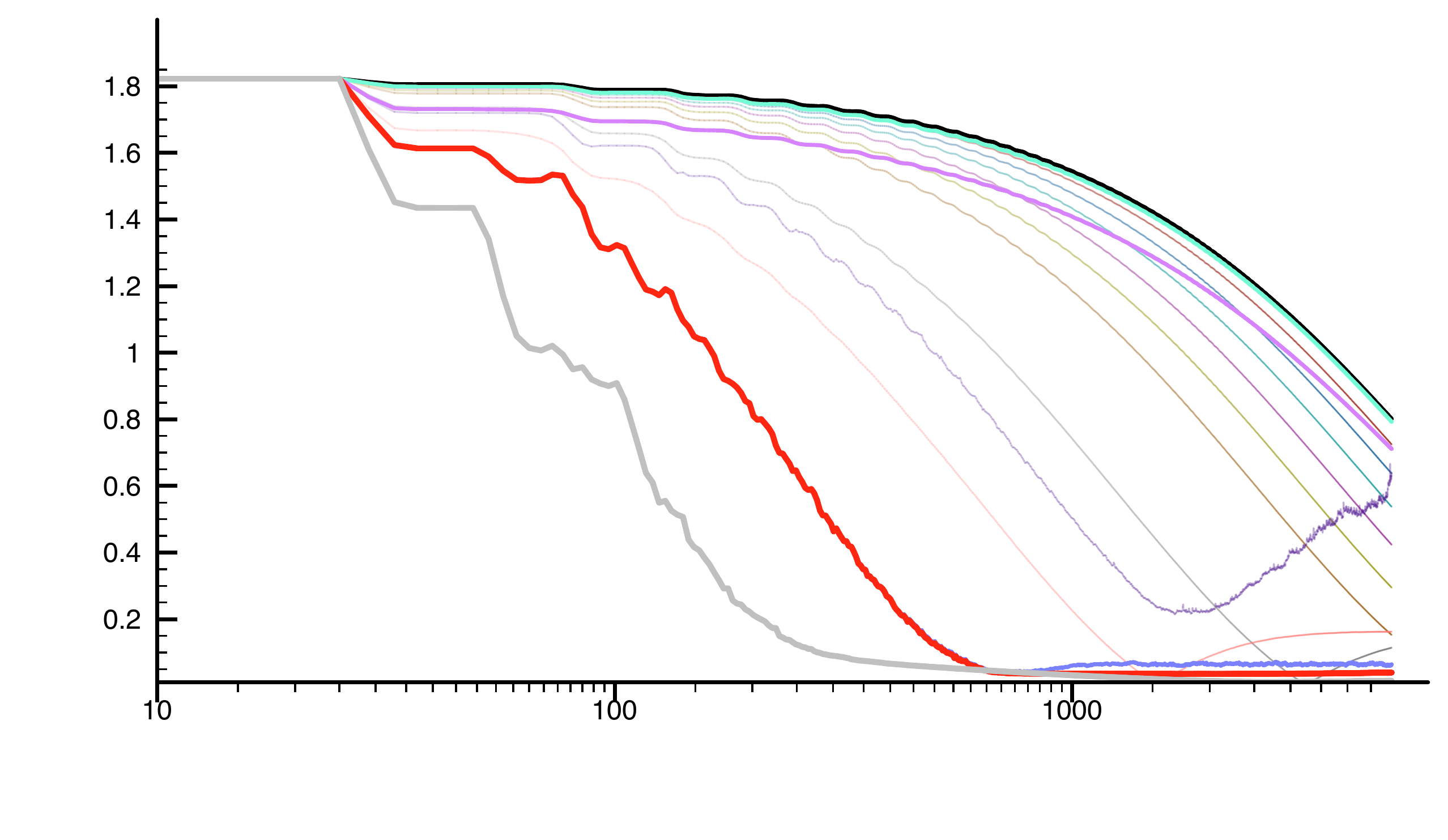}\\
\addlabel{$\lambda$ value} &  \includegraphics[width=\gwidth]{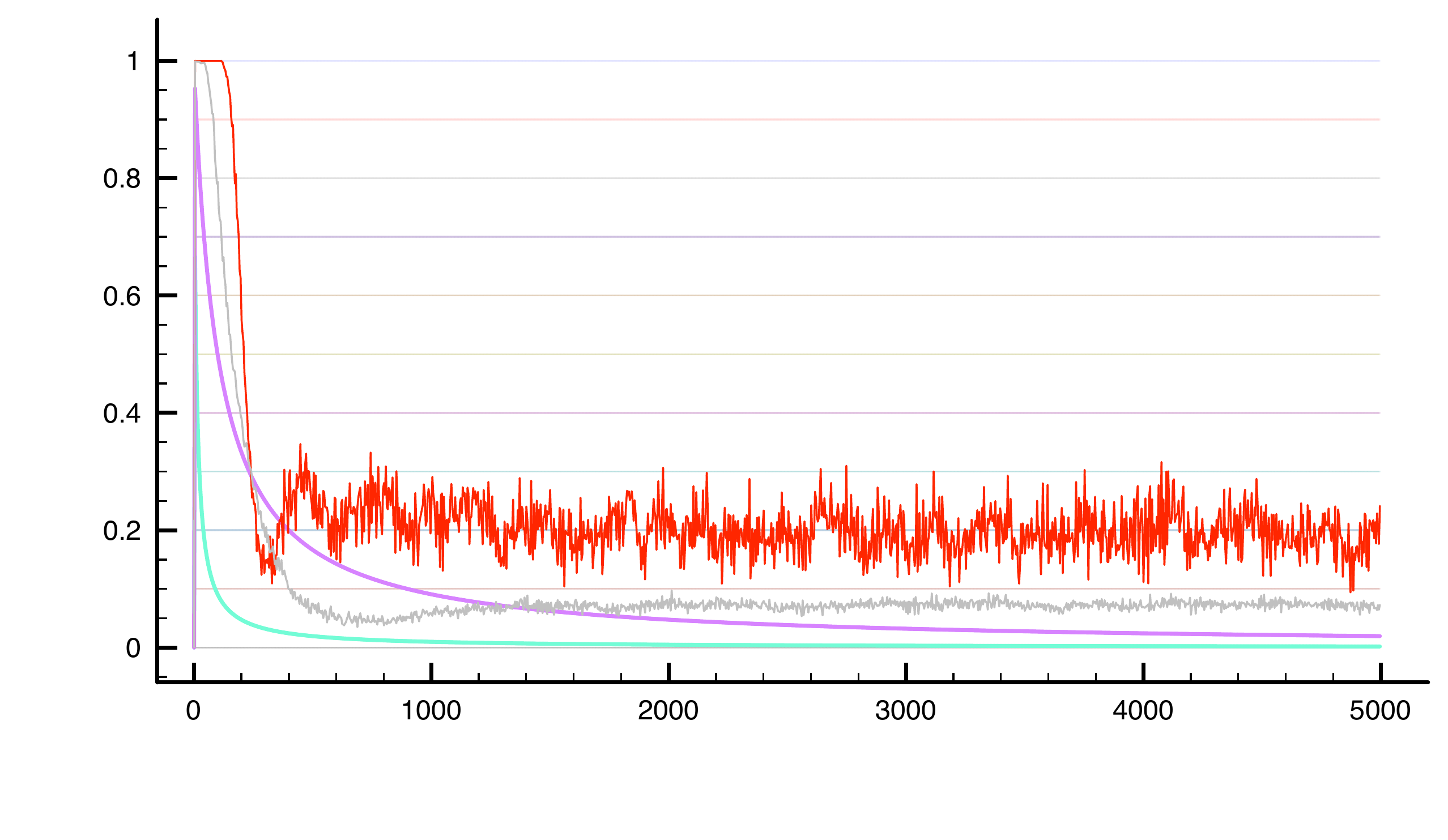} & \includegraphics[width=\gwidth]{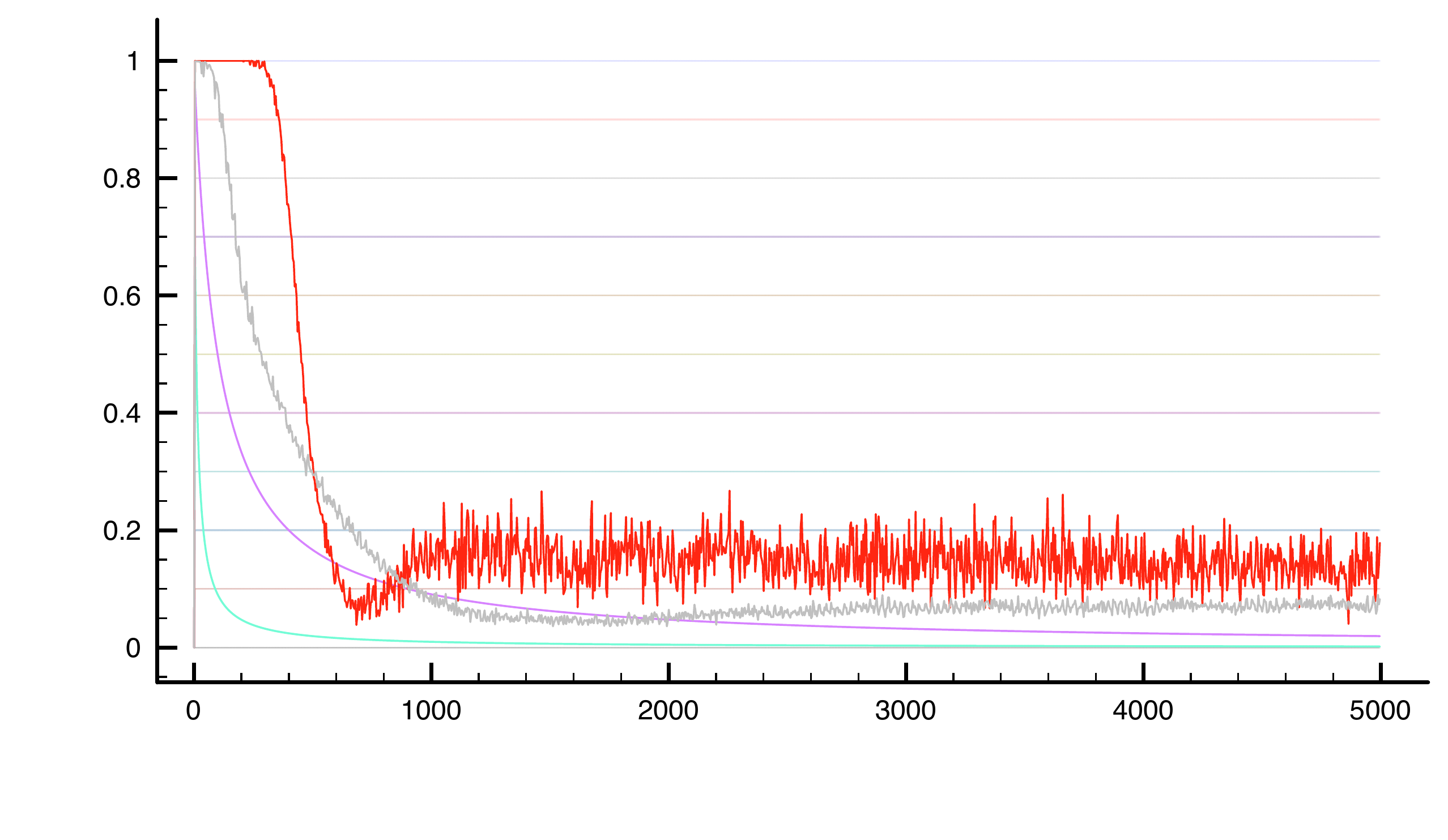} \hspace{-2.0cm}&    \includegraphics[width=\gwidth]{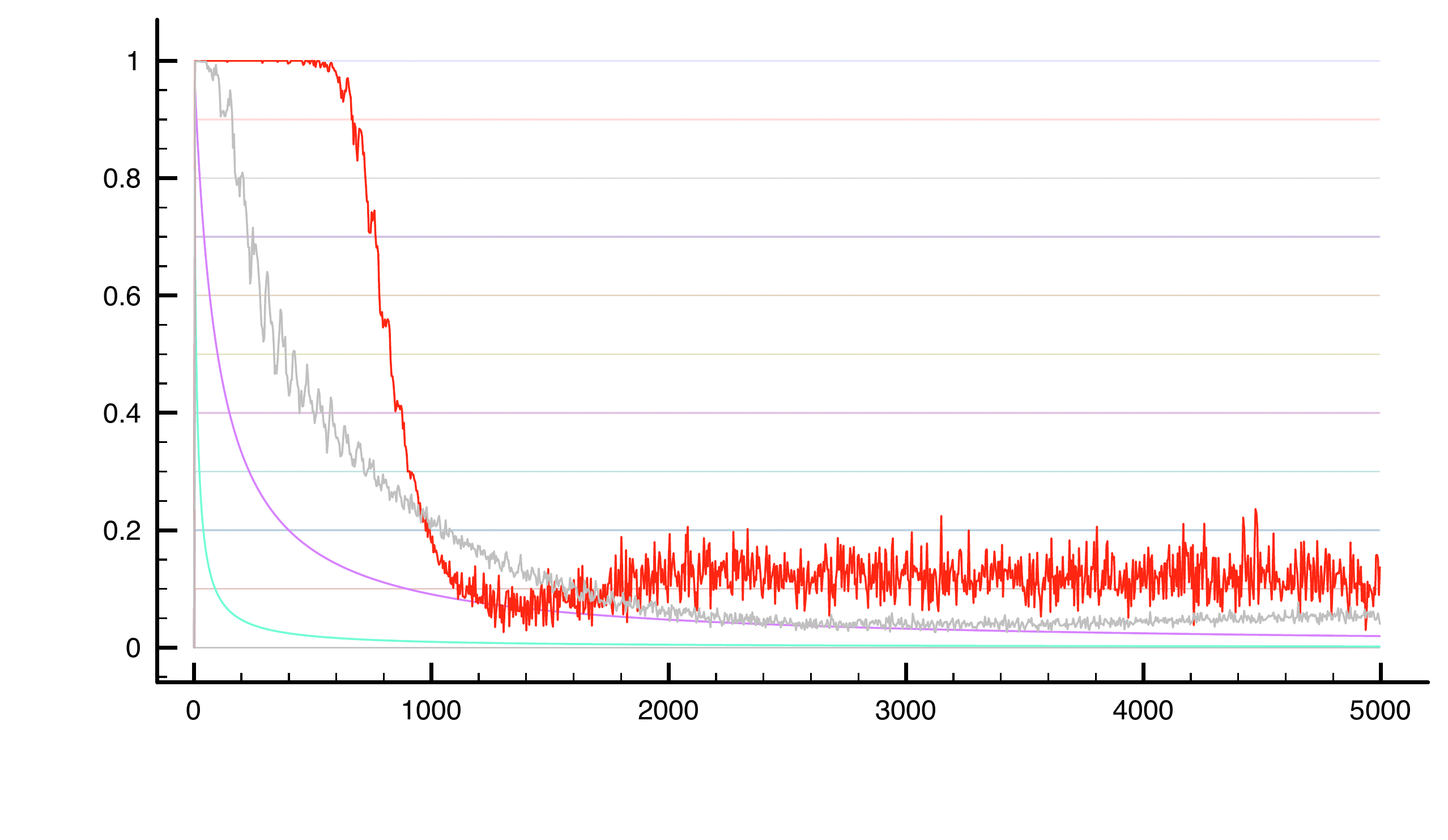}\\
  \vspace{-.2cm}
& Chain 10 & Chain 25  & Chain 50
 \end{tabular}

 \caption{{\bf On-policy} learning performance of GTD($\lambda$) with different methods of adapting $\lambda$, for three sizes in the ringworld domain. The first row of results plots the mean squared value error verses time for several approaches to adapting $\lambda$: the optimal $\lambda$-greedy algorithm, our approximation algorithm, several fixed values of $\lambda$, and two fixed decay schedules. The second row of results plots the $\lambda$ values over-time used by each algorithm. We highlight early learning by taking the log of y axis, as the algorithms should be able to exploit the low variance to learn quickly. For off-policy learning, there is more variance due to importance sampling and so we prefer long-term stability. We therefore highlight later learning using the log of the x-axis.}\label{figure_on_learning}
 \end{figure*}
 \begin{figure*}
\centering
\hspace*{-1.2cm}
\begin{tabular}{ccc}
 \addlabel{Error} & \includegraphics[width=\gwidth]{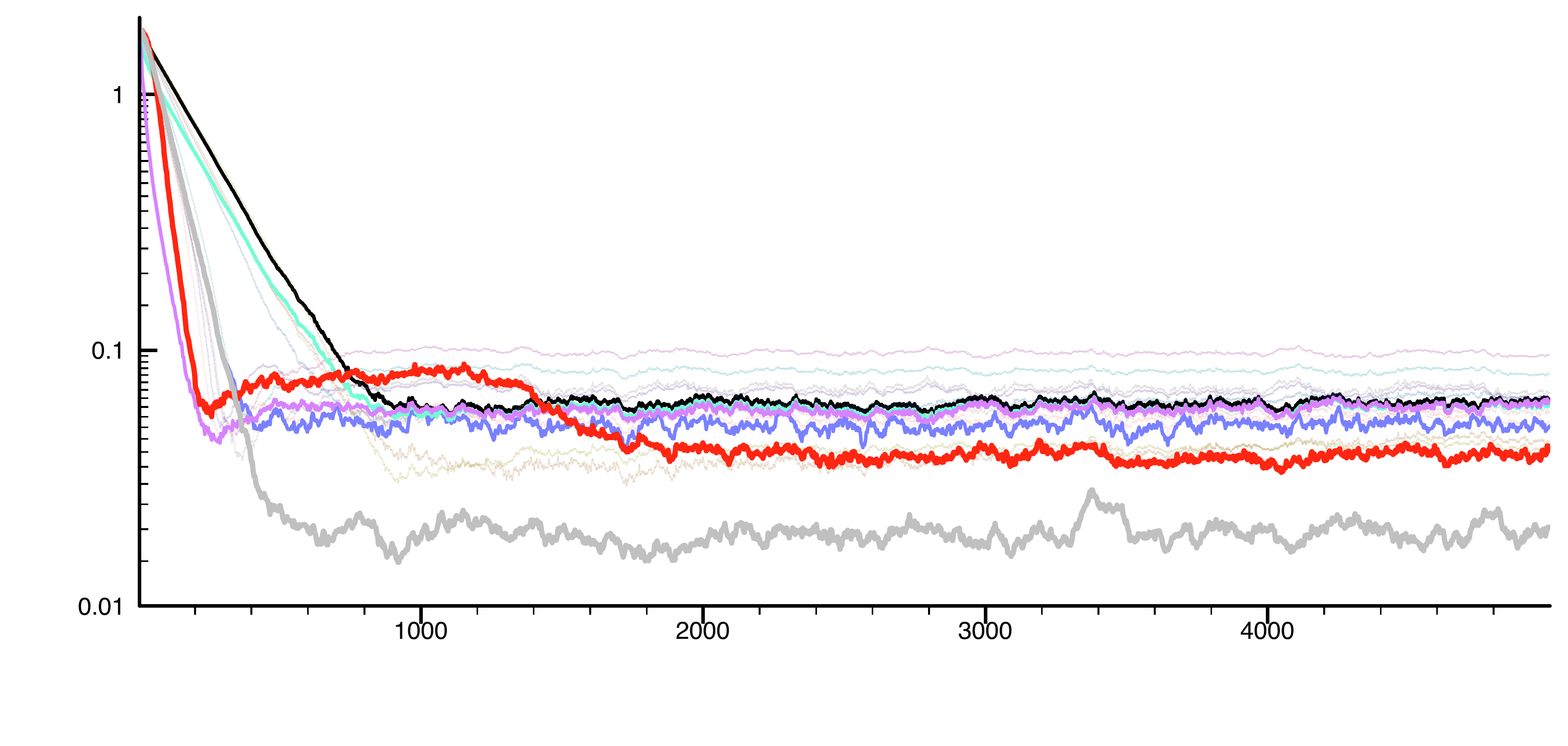} \hspace{-2.0cm}& \includegraphics[width=\gwidth]{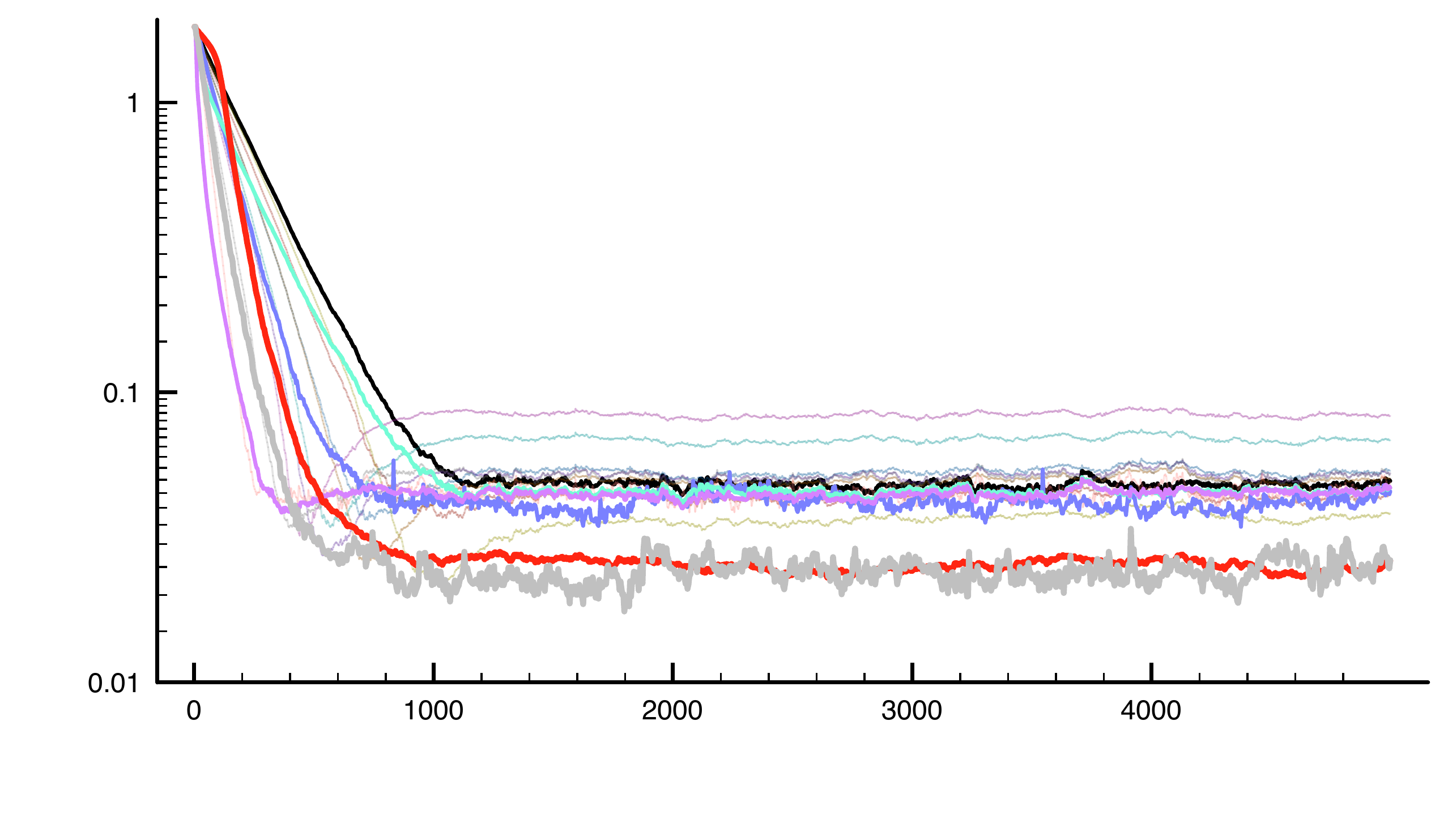} \\ 
  \addlabel{$\lambda$ value} \hspace{-2.0cm}& \includegraphics[width=\gwidth]{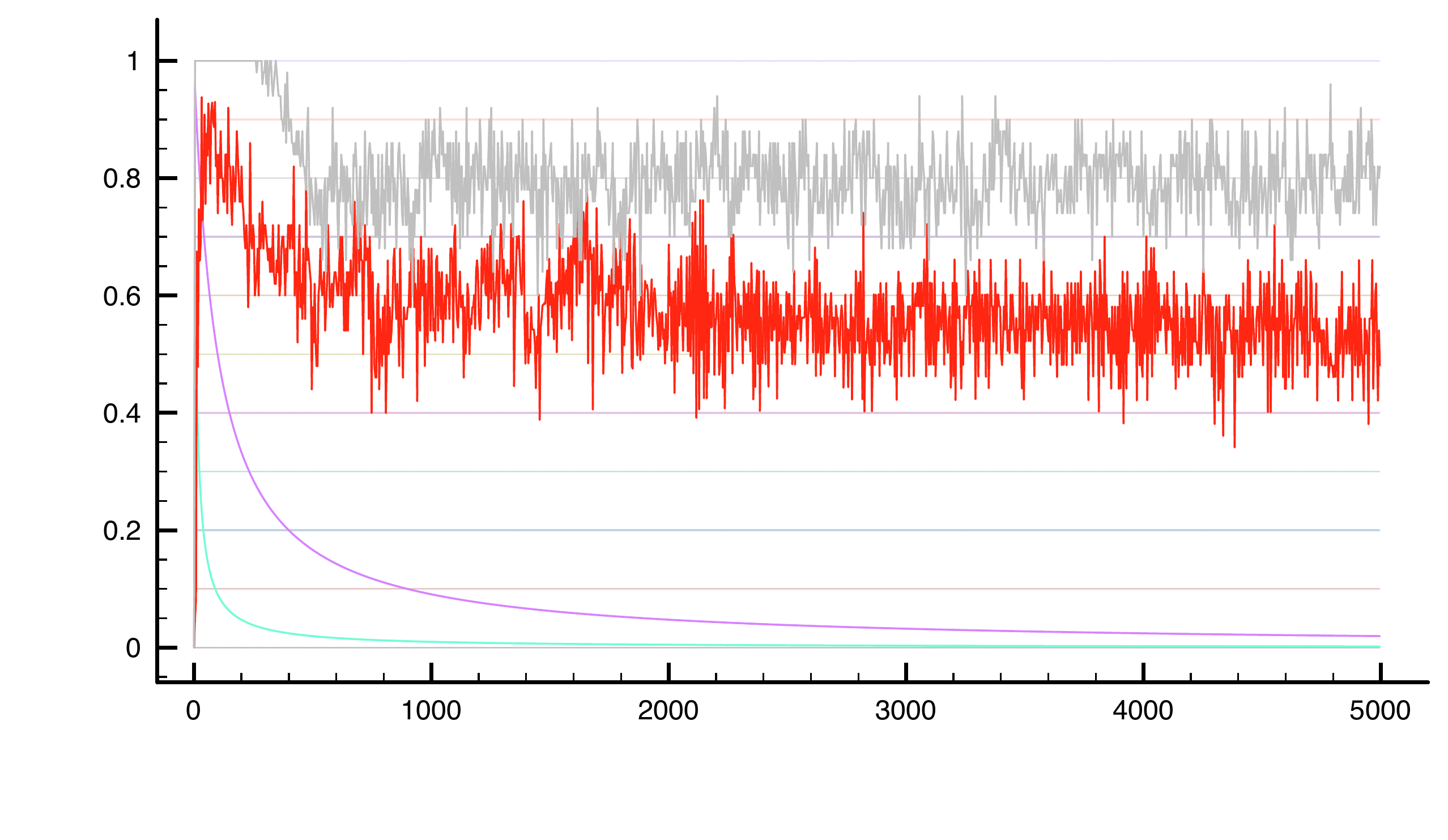} \hspace{-2.0cm}& \includegraphics[width=\gwidth]{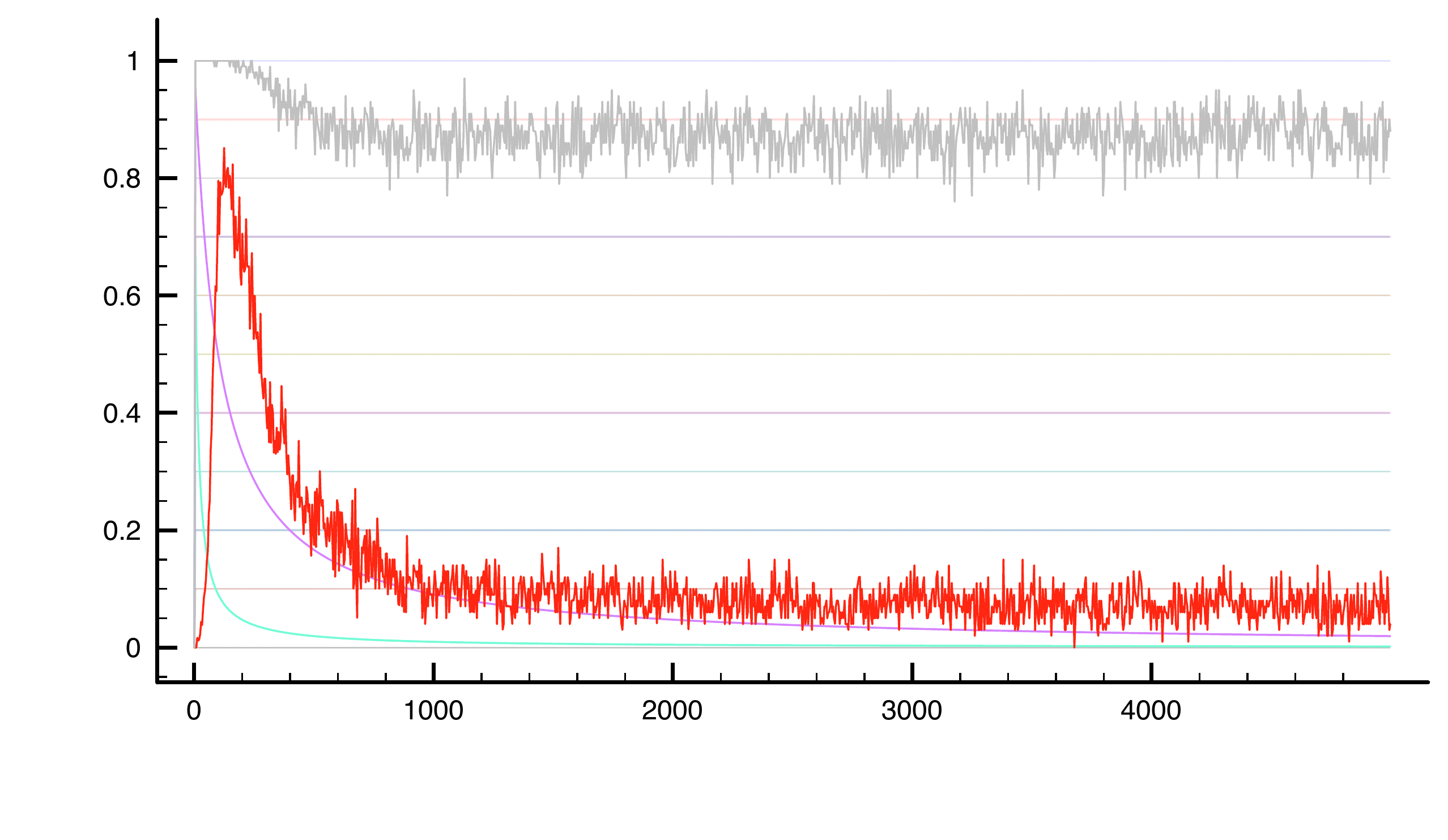} \\
\hspace{-2.0cm}& $\pi(s,\text{right}) = 0.95, \mu(s,\text{right})=0.85$ \hspace{-2.0cm}& $\pi(s,\text{right}) = 0.95, \mu(s,\text{right})=0.75$ 
   \vspace{-.2cm}
 \end{tabular}
 \caption{{\bf Off-policy} learning performance of GTD($\lambda$) with different methods of adapting $\lambda$, for two different configurations of $\pi$ and $\mu$ in the size 10 ringworld domain. The first row of results plots the mean squared value error verses time for several approaches to adapting $\lambda$. In off-policy learning, there is more variance due to importance sampling and so we prefer long-term stability. We therefore highlight later learning using the log of the x-axis. The second row of results plots the $\lambda$ values over-time used by each algorithm. }\label{figure_off_learning}
 \end{figure*}
  \begin{figure*}
\centering
\hspace*{-1.2cm}
\begin{minipage}[b]{.35\textwidth}
\centering
  \includegraphics[width=\textwidth]{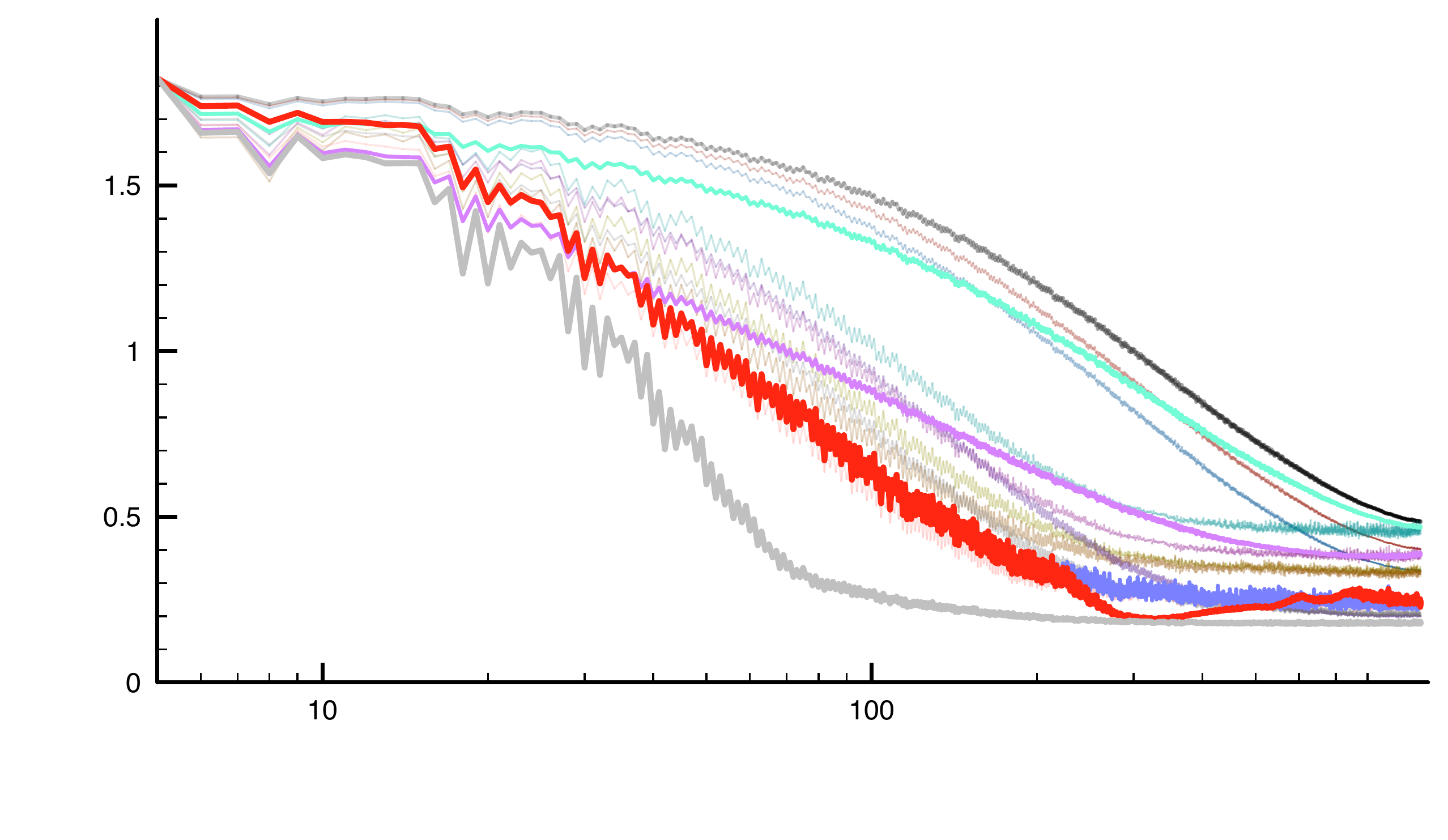} 
  \end{minipage}
  \begin{minipage}[b]{.35\textwidth}
\centering
 \includegraphics[width=\textwidth]{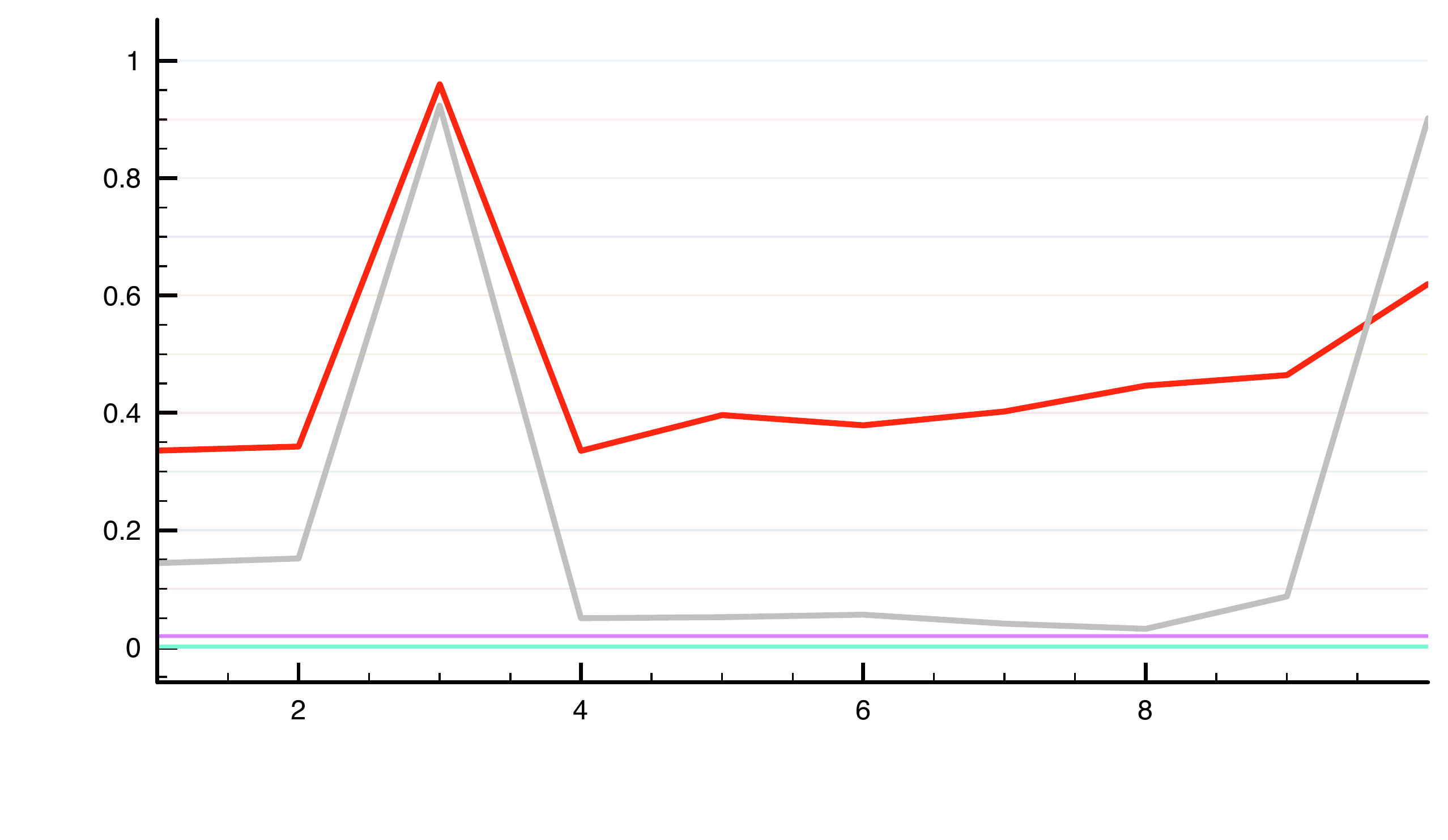}
 \end{minipage}
\begin{minipage}[b]{.25\textwidth}
\centering
 \caption{A closer examination of the effects of state aliasing on the choice of state-based $\lambda$, for the length 10 chain with $\gamma = 0.95$, under on-policy sampling. State three is aliased with the last state. The left graph shows the learning curves over 5000 steps, and the right graph shows the learned $\lambda$ function in each state at the end of the run. }\label{figure_lambda}
 \end{minipage}
 \end{figure*}
 
 We investigate \lamname\
 in a ring-world, under both on and off-policy sampling.
 This ring-world was previously introduced as a suitable domain for investigating $\lambda$ \cite{kearns2000bias}.
 We varied the length of the ring-world from
 $N = 10, 25, 50$. 
 The reward is zero in every state,
 except
 for two adjoining states that have +1 and -1 reward,
 and are terminal states.
 The agent is teleported back to the middle
 of the ring-world upon termination.
 The target policy is to take action ``right" with 95\% probability,
 and action ``left" with 5\% probability. 
The feature representation is a tabular encoding of states with binary identity vectors, but we also examine the effects of aliasing state
values to simulate poor generalization: a common case where the true
value function cannot be represented. 
The length of the experiment is a function of the problem size, $N \times 100$, proportionally scaling 
the number of samples for 
longer problem instances.

We compared to fixed values of $\lambda = 0, 0.1, \ldots, 0.9, 1.0$ 
and to two time-decay schedules, $\lambda_t = 10/(10+t), \lambda_t = 100/(100+t)$,
which worked well compared to several other tested settings.
The discount factor is $\gamma = 0.99$ for the on-policy chains and 0.95 for the off-policy chains.
We include the optimal \lamname, which computes $\hvar$ and $\hsq$ using closed form solutions,
defined by their respective Bellman operators.
For \lamname, the initialization was $\hvar = 0.0$ and 
$\herr = \frac{Rmax}{1-\gamma} \times \onevec$ for on-policy, to match the rule-of-thumb of initializing with high lambdas,
and the opposite for off-policy, to match the rule-of-thumb of more caution in off-policy domains. 
This max return value is a common choice for optimistic initialization, and prevented inadvertent evaluator bias
by overly tuning this parameter. We fixed the learning-rate parameter for \lamname\ to be equal to equal to $\alpha$ used in learning the value function ($\w$), again to demonstrate performance
 in less than ideal settings. Sweeping the step-size for \lamname\ would
 improve performance. 

The performance results in on and off-policy are summarized in Figure \ref{figure_on_learning} and \ref{figure_off_learning}.
We report the absolute value error compared to the true value function,
which is computable in this domain for each of the settings. 
We average over 100 runs, and report the results for the best
parameter settings for each of the algorithms with 12 values of 
$\alpha \in \{0.1\times 2^j|j=-6,-6,...,5,6\}$, 11 values of $\eta \in \{2^j|j=-16,-8,-4,\ldots, 4, 8, 16\}$ ($\alpha_\h = \alpha\eta$).

In general, we find that \lamname\ works
well across settings.
The optimal \lamname\ consistently performs the best, indicating
the merit of the objective. Estimating $\hvar$ and $\hsq$ typically
cause \lamname\ to perform more poorly, indicating an opportunity to improve these algorithms
to match the performance of the optimal, idealistic version. 
In particular, we did not optimize the meta-parameters in \lamname.
For the fixed values of $\lambda$ and decay schedules,
we find that they can be effective for specific instances, but
do not perform well across problem settings. In particular, the fixed decay schedules
settings appear to be un-robust to an increasing chain length
and the fixed $\lambda$ are not robust to the change from on-policy to off-policy. 

We also examined the $\lambda$ value selected by our \lamname\
algorithm plotted against time. For tabular features, $\lambda$ should converge to zero over-time,
since the value function can be approximated and so, at some point, no
bias is introduced by using $\lambda_t = 0$, but variance is reduced.
The algorithm does converge to a state-based $\lambda$ ($\lambda(s) \approx 0$ for all states), 
which was one of our goals to ensure we have a well-defined fixed point. 
Second, for aliased features, we expect that the final per-state $\lambda$
should be larger for the states that have been aliased.
The intuition is that one should not bootstrap on the values
of these states, as that introduces bias.
We demonstrate that this does indeed
occur, in Figure \ref{figure_lambda}. 
As expected, the  \lamname\ is
more robust to state aliasing, compared to the fixed strategies. State aliasing provides a concrete example of when we have less confidence in $\hat{v}$ in specific states, and an effective strategy to mitigate this situation is to set $\lambda$ high in those states.  

\section{Conclusion and discussion}

In this paper, we have proposed
the first linear-complexity $\lambda$ adaptation algorithm
for incremental, off-policy reinforcement learning.
We proposed an objective to greedily trade-off bias
and variance, and derived an efficient algorithm to obtain
the solution to this objective. We demonstrate
the efficacy of the approach in an on-policy and off-policy setting,
versus fixed values of $\lambda$ and time-decay heuristics.

This work opens up many avenues for efficiently selecting $\lambda$,
by providing a concrete greedy objective. 
One important direction is to extend the above analysis to use the recently introduced true-online traces \cite{vanseijen2014true};
here we focused on the more well-understood $\lambda$-return, but
the proposed objective is not restricted to that setting. 
There are also numerous
future directions for improving optimization of this objective. 
We used GTD($\lambda=1$) to learn the second moment of the $\lambda$-return inside \lamname.
Another possible direction is to simply explore using $\lambda<1$
or even least-squares methods to improve estimation inside \lamname.

There are also opportunities to modify the objective to consider variance into the future differently. The current objective is highly cautious, in that it assumes that only $\lambda_{t+1}$ can be modified, and assumes the agent will be forced to use the unbiased $\lambda_{t+i} = 1$ for all future time steps. As a consequence, the algorithm will often prefer to set $\lambda_{t+1}$ small, as future variance can only be controlled by the setting of $\lambda_{t+1}$.  
A natural relaxation of this strictly greedy setting is to recognize that the agent, in fact, has control over all future $\lambda_{t+i}, i>1$. 
The agent could remain cautious for some number of steps, assuming $\lambda_{t+i}$ will likely be high for $1<i\le k$ for some horizon $k$. However, we can assume that after $k$ steps, $\lambda_{t+i}$ we be reduced to a small value, cutting off the traces and mitigating the variance. 
The horizon $k$ could be encoded with an aggressive discount; multiplying the current discount $\gambar_{t+1}$ by a value less than 1. 
Commonly, a multiplicative factor less than 1 indicates a horizon of $\frac{1}{1-\text{factor}}$; for example, $\frac{1}{1-0.8} = 5$ gives a horizon of $k=5$. The variance term in the objective in \eqref{eq_lambda} could be modified to include this horizon, providing a method to incorporate the level of caution. 

Overall, this work takes a small step toward the goal of
automatically selecting the trace parameter, and thus one step closer toward parameter-free black-box application of reinforcement learning with many avenues for future research.
\\
{\small
{\bf Acknowledgements}\\
We would like to thank David Silver for helpful discussions
and the reviewers for helpful comments.
}
\bibliographystyle{abbrv}
\bibliography{aamas2015.bib}

\appendix

Recall that for squared return $\glambdasq{t}$, we use $\glambdasqbar_t$ to indicate
the $\lambar$-squared-return.
\begin{align*}
\glambdasqbar_{t} &\defeq \rbar_{t+1} + \gambar_{t+1} \left( (1-\lambar_{t+1}) \xvec_{t+1}^\top \hvar + \lambar_{t+1} \glambdasqbar_{t+1}\right)\\
\deltalambar{t} &\defeq \glambdasqbar_{t} -  \x_t^\top \hsq\\
\deltabar_{t} &\defeq \rbar_{t+1} + \gambar_{t+1} \xvec_{t+1}^\top \hvar  -  \x_t^\top \hsq\\
\rbar_{t+1} &\defeq \rho_t^2 \gbar_{t}^2 + 2 \rho_t^2 \gamma_{t+1} \lambda_{t+1} \gbar_{t} \glambda{t+1}\\
\gambar_{t+1} &\defeq \rho_t^2 \gamma_{t+1}^2 \lambda_{t+1}^2\\
\gbar_{t} &\defeq R_{t+1} + \gamma_{t+1} (1-\lambda_{t+1}) \x_{t+1}^\top \w 
.
\end{align*}
If $\lambar = 1$, then $\glambdasqbar_t = \glambdasq{t}$.
Note that the weights $\w$ in $\gbar_t$ are the weights for the main estimator that defines the $\lambda$-return,
not the weights $\hvar$ that estimate the second moment of the $\lambda$-return.

\noindent
\textbf{Theorem \ref{thm_main}}
For a given $\lambar: \Ss \rightarrow [0,1]$,
\begin{align*}
\E[\deltalambar{t} \x_t] &= \E[\deltabar_{t} \ztrace_{t}]
\end{align*}
where
\begin{align*}
\gtwotrace{t} &\defeq \x_t+ \gambar_{t} \lambar_{t}\gtwotrace{t-1}
\end{align*}
%
%
\begin{proof}
As in other TD algorithms \cite{maei2011gradient}, we use
index shifting to obtain an unbiased estimate of future values using traces.

\begin{align*}
\glambdasqbar_{t} 
&= \rbar_{t+1} + \gambar_{t+1} \left( (1-\lambar_{t+1}) \xvec_{t+1}^\top \hvar + \lambar_{t+1} \glambdasqbar_{t+1} \right) \\
&= \rbar_{t+1} + \gambar_{t+1} \xvec_{t+1}^\top \hvar - \gambar_{t+1} \lambar_{t+1} \xvec_{t+1}^\top \hvar + \gambar_{t+1} \lambar_{t+1} \glambdasqbar_{t+1}\\
&= \rbar_{t+1} + \gambar_{t+1} \xvec_{t+1}^\top \hvar + \gambar_{t+1} \lambar_{t+1} \deltalambar{t+1}\\
&= \rbar_{t+1} + \gambar_{t+1} \xvec_{t+1}^\top \hvar + \gambar_{t+1} \lambar_{t+1} \deltalambar{t+1}
\end{align*}
Therefore, for $\deltabar_{t} = \rbar_{t+1} + \gambar_{t+1} \xvec_{t+1}^\top \hvar  -  \x_t^\top \hsq$,
\begin{align*}
\E[\deltalambar{t} \x_t] &= \E[(\glambdasqbar_{t}   -  \x_t^\top \hsq) \x_t]\\
&= \E[(\deltabar_{t} + \gambar_{t+1} \lambar_{t+1} \deltalambar{t+1}) \x_t]\\
&= \E[\deltabar_{t} \x_t] + \E[\gambar_{t} \lambar_{t} \deltalambar{t} \x_{t-1}]\\
&= \E[\deltabar_{t} \x_t] + \E[\gambar_{t} \lambar_{t} (\deltabar_{t} + \gambar_{t+1} \lambar_{t+1} \deltalambar{t+1}) \x_{t-1}]\\
&= \E[\deltabar_{t} (\x_t + \gambar_{t} \lambar_{t} \x_{t-1}) ] + \E[\gambar_{t} \lambar_{t} \gambar_{t+1} \lambar_{t+1} \deltalambar{t+1} \x_{t-1}]\\
&= \E[\deltabar_{t} (\x_t + \gambar_{t} \lambar_{t} \x_{t-1}) ] + \E[\gambar_{t-1} \lambar_{t-1} \gambar_{t} \lambar_{t} \deltalambar{t} \x_{t-2}]\\
&= \ldots \\
&= \E[\deltabar_{t} \ztrace_{t} ] 
\end{align*}
%
\end{proof}

To enable the above recursive form, we need to ensure that 
$\rbar_{t+1}$ is computable on each step, given $S_t, S_{t+1}$.
We characterize the expectation of $\rbar_{t+1}$,
using the current unbiased estimate $\x_t^\top\herr$ of  $\E[\glambda{t+1} | S_{t+1}]$,
providing a way to obtain an unbiased sample of $\rbar_{t+1}$ on each step. 
\begin{proposition}
\begin{align*}
\E[\rbar_{t+1} | S_t, S_{t+1}] &=  \E[\rho_t^2 \gbar_{t}^2 | S_t, S_{t+1}] \\
&+ 2 \E[\rho_t^2 \gamma_{t+1} \lambda_{t+1} \gbar_{t} |S_t, S_{t+1}] \E[\glambda{t+1} | S_{t+1}]
\end{align*}
\end{proposition}
\begin{proof}
\begin{align*}
\E[\rbar_{t+1} | S_t, S_{t+1}] 
&=  \E[\rho_t^2 \gbar_{t}^2 + 2 \rho_t^2 \gamma_{t+1} \lambda_{t+1} \gbar_{t} \glambda{t+1} | S_t, S_{t+1}]\\
&=  \E[\rho_t^2 \gbar_{t}^2 | S_t, S_{t+1}] \\
&+ 2 \E[\rho_t^2 \gamma_{t+1} \lambda_{t+1} |S_t, S_{t+1}] \E[ \gbar_{t} \glambda{t+1} | S_t, S_{t+1}]
\end{align*}
Now because we assumed independent, zero-mean noise in the reward, i.e. $R_{t+1} = \E[R_{t+1} | S_t, S_{t+1}] + \epsilon_{t+1}$
for independent zero-mean noise $\epsilon_{t+1}$,
because we assume that $\x_{t+1}$ is completely determined by $S_{t+1}$
and because \\$\E[\glambda{t+1} | S_t, S_{t+1}] = \E[\glambda{t+1} | S_{t+1}]$,
we obtain
\begin{align*}
&\E[ \gbar_{t} \glambda{t+1} | S_{t+1}]\\
&= \E[ (R_{t+1} + \gamma_{t+1} (1-\lambda_{t+1}) \x_{t+1}^\top \w )\glambda{t+1} | S_t, S_{t+1}]\\
&= \E[ R_{t+1} | S_t, S_{t+1}] \E[\glambda{t+1} | S_{t+1}] \\
&+ \E[\gamma_{t+1} (1-\lambda_{t+1}) \x_{t+1}^\top \w | S_t,S_{t+1}] \E[\glambda{t+1} | S_{t+1}]
\end{align*}
completing the proof.
\end{proof}

This previous result requires that $\rbar_{t+1}$ can be given
to the algorithm on each step. 
The result in Theorem \ref{thm_main} requires an
an estimate of $\E[\glambda{t+1} | S_{t+1}]$, 
to define the reward; such an estimate, however, 
may not always be available. 
We characterize the expectation $\E[\deltalambar{t} \x_t]$
in the next theorem, when such an estimate is not provided. 
For general trace functions $\lambda$,
this estimate requires a matrix to be stored and updated; however, for $\lambda=1$
as addressed in this work, Corollary \ref{corollary_mse} indicates
that this simplifies to a linear-space and time algorithm.

\newcommand{\thirdtermspace}{\makebox[\widthof{$\ = \E[\rho_t^2 \gbar_t^2 \x_t] $}]{}}
\newcommand{\smallspace}{ \ \ \ \ \ }

\noindent
\textbf{Theorem \ref{thm_mse}}
\begin{align*}
\E[\rbar_{t+1}  \ztrace_t] = E[\rho_{t+1}^2\gbar_t^2 \gtwotrace{t} ] + 2 E[\rho_{t+1}^2\gbar_t (\gvectracer{t} + \gvectracex{t} \w_t)] 
\end{align*}
where
\begin{align*}
\gvectracer{t} &\defeq \rho_t \gamma_{t} \lambda_{t} (R_{t} \gtwotrace{t-1}  + \gvectracer{t-1})\\
\gvectracex{t} &\defeq \rho_t \gamma_{t} \lambda_{t} (\gamma_{t} (1-\lambda_{t}) \gtwotrace{t-1} \x_{t}^\top   + \gvectracex{t-1})
\end{align*}
%
%
\begin{proof}
\begin{align*}
\E[\rbar_{t+1}  \ztrace_t] = E[\rho_{t}^2\gbar_t^2 \gtwotrace{t} ] + 2 E[\rho_{t}^2\gamma_{t+1} \lambda_{t+1}\gbar_t \glambda{t+1} \ztrace_t] 
\end{align*}
Now let us look at this second term which still contains $\glambda{t}$.
\begin{align*}
&\E[\gamma_{t} \lambda_{t} \glambda{t} \gbar_{t-1} \rho_{t-1}^2\gtwotrace{t-1}] \\
&=\E[\rho_t \gamma_{t} \lambda_{t} (\gbar_t + \gamma_{t+1} \lambda_{t+1} \glambda{t+1})\gbar_{t-1}  \rho_{t-1}^2\gtwotrace{t-1}] \\
&=\E[\rho_t \gamma_{t} \lambda_{t} \gbar_t \gbar_{t-1} \rho_{t-1}^2\gtwotrace{t-1}] \\
&\smallspace+ \E[\rho_{t-1} \gamma_{t-1} \lambda_{t-1} \gamma_{t} \lambda_{t} \glambda{t}\gbar_{t-2}\rho_{t-2}^2\gtwotrace{t-2}] \\
&=\E[\rho_t \gamma_{t} \lambda_{t} \gbar_t \gbar_{t-1} \rho_{t-1}^2\gtwotrace{t-1}] \\
&+ E[ \rho_t \gamma_{t} \lambda_{t} \rho_{t-1} \gamma_{t-1} \lambda_{t-1} (\gbar_t + \gamma_{t+1} \lambda_{t+1} \glambda{t+1}) \gbar_{t-2}\rho_{t-2}^2\gtwotrace{t-2}]  \\
&=\E[\rho_t \gamma_{t} \lambda_{t} \gbar_t \gbar_{t-1} \rho_{t-1}^2\gtwotrace{t-1}] \\
&\smallspace + \E[\rho_t \gamma_{t} \lambda_{t} \rho_{t-1} \gamma_{t-1} \lambda_{t-1} \gbar_t \gbar_{t-2}\rho_{t-2}^2\gtwotrace{t-2}]  \\
&\smallspace + \E[\rho_{t-1} \gamma_{t-1} \lambda_{t-1} \rho_{t-2} \gamma_{t-2} \lambda_{t-2} \gamma_{t} \lambda_{t} \glambda{t} \gbar_{t-3} \rho_{t-3}^2\gtwotrace{t-3}]  \\
&= \ldots\\
&=\E[\rho_t \gamma_{t} \lambda_{t} \gbar_t (\gbar_{t-1} \rho_{t-1}^2\gtwotrace{t-1} + \rho_{t-1} \gamma_{t-1} \lambda_{t-1} \gbar_{t-2}\rho_{t-2}^2\gtwotrace{t-2} + \ldots) ] 
.
\end{align*}
Each $\gbar_{t-i}$ into the past should use the current (given) weight vector $\w$,
to get the current value estimates in the return. 
Therefore, instead of keeping a trace of $\gbar_{t-i} = r_{t-i+1} + \gamma_{t-i+1} (1-\lambda_{t-i+1}) \x_{t-i+1}^\top \w_{t-i+1}$, 
we will keep a trace of feature vectors
and a trace of the reward
\begin{align*}
\gvectracer{t} &= \rho_t \gamma_{t} \lambda_{t} (r_{t} \rho_{t-1}^2 \gtwotrace{t-1}  + \gvectracer{t-1})\\
\gvectracex{t} &=  \rho_t \gamma_{t} \lambda_{t} (\gamma_{t} (1-\lambda_{t})  \rho_{t-1}^2\gtwotrace{t-1} \x_{t}^\top  + \gvectracex{t-1})
.
\end{align*}
Because 
\begin{align*}
\gvectracex{t} \w &= \rho_t\gamma_{t} \lambda_{t} (\gamma_{t} (1-\lambda_{t})\gtwotrace{t-1} \x_t^\top \w  + \gvectracex{t-1} \w)
\end{align*}
we get
\begin{align*}
&\gvectracer{t} + \gvectracex{t} \w\\
&= \rho_t\gamma_{t} \lambda_{t} (\gbar_{t-1} \rho_{t-1}^2 \gtwotrace{t-1} + \rho_{t-1} \gamma_{t-1} \lambda_{t-1} \gbar_{t-2} \rho_{t-2}^2\gtwotrace{t-2} + \ldots)
\end{align*}
giving
\begin{align*}
\E[\gamma_{t} \lambda_{t} \glambda{t} \gbar_{t-1} \rho_{t-1}^2\gtwotrace{t-1}] 
&= \E[\gbar_t (\gvectracer{t} + \gvectracex{t} \w)]
.
\end{align*}
Finally, we can put this all together to get
\begin{align*}
\E[\rbar_{t+1}  \ztrace_t] 
&= E[\rho_{t}^2\gbar_t^2 \gtwotrace{t} ] + 2 E[\rho_{t}^2\gamma_{t+1} \lambda_{t+1}\gbar_t \glambda{t+1} \ztrace_t] \\
&= E[\rho_{t}^2\gbar_t^2 \gtwotrace{t} ] + 2 E[\rho_{t-1}^2\gamma_{t} \lambda_{t}\gbar_{t-1} \glambda{t} \ztrace_{t-1}] \\
&= E[\rho_{t}^2\gbar_t^2 \gtwotrace{t} ] + 2 \E[\gbar_t (\gvectracer{t} + \gvectracex{t} \w)] 
\end{align*}
%
\par \vspace{-0.3cm}
\end{proof}


\noindent
\textbf{Remark:}
One can derive that $\E[\glambdasq{t} \x_t] = \E[\rbar_{t+1}  \ztrace_t]$;
this provides a way to use an LMS update rule for the above, without bootstrapping
and corresponds to minimizing the mean-squared error using samples of returns.
However, as has previously been observed, 
we found this performed
more poorly than using bootstrapping. 

\vspace{1.0cm}

\noindent
\textbf{Derivation of the gradient of VTD}

\begin{align*}
\tfrac{1}{2} \nabla \varMSPBE(\hsq) 
&= \tfrac{1}{2} \nabla (\E[\deltalambar{t} \x_t]^\top \E[\x_t \x_t]^{-1} \E[\deltalambar{t} \x_t])\\
&=  \nabla \E[\deltalambar{t} \x_t]^\top (\E[\x_t \x_t]^{-1} \E[\deltalambar{t} \x_t])
.
\end{align*}
First we characterize this gradient $\nabla \E[\deltalambar{t} \x_t]^\top$.
\begin{align*}
\nabla \E[\deltalambar{t} \x_t]^\top 
 &= \nabla \E[ \deltabar_t \gtwotrace{t}]\\
  &=  \E[(\gambar_{t+1} \xvec_{t+1} - \xvec_t) \gtwotrace{t}^\top ]
  .
\end{align*}
This product can be simplified using
\begin{align*}
\E[\xvec_t \gtwotrace{t}^\top ] 
&= \E[\xvec_t (\xvec_t + \gambar_t \lambar_t \gtwotrace{t-1})^\top ]\\ 
&= \E[\xvec_t \xvec_t^\top ] + \E[\xvec_{t+1} \gambar_{t+1} \lambar_{t+1}  \gtwotrace{t}^\top ]\\ 
&= \E[\xvec_t \xvec_t^\top ] + \E[ \gambar_{t+1} \lambar_{t+1} \xvec_{t+1}\gtwotrace{t}^\top ]
\end{align*}
giving 
\begin{align*}
\nabla \E[\deltalambar{t} \x_t]^\top 
&= \E[\gambar_{t+1} (1-\lambar_{t+1}) \xvec_{t+1} \gtwotrace{t}^\top] - \E[\xvec_t \xvec_t^\top ] 
.
\end{align*}
Therefore, the gradient reduces to
\begin{align*}
&-\tfrac{1}{2} \nabla \varMSPBE(\hsq) \\
&=  (\E[\xvec_t \xvec_t^\top ]  - \E[\gambar_{t+1} (1-\lambar_{t+1}) \xvec_{t+1} \gtwotrace{t}^\top] ) (\E[\x_t \x_t]^{-1} \E[\deltalambar{t} \x_t])\\
&=  \E[\deltalambar{t} \x_t]  - \E[\gambar_{t+1} (1-\lambar_{t+1}) \xvec_{t+1} \gtwotrace{t}^\top]\E[\x_t \x_t]^{-1} \E[\deltalambar{t} \x_t]
.
\end{align*}

\end{document}